\documentclass{vgtc}                          

\ifpdf
  \pdfoutput=1\relax                   
  \usepackage{graphicx}                
  \DeclareGraphicsExtensions{.pdf,.png,.jpg,.jpeg} 
\else
  \usepackage{graphicx}                
  \DeclareGraphicsExtensions{.eps}     
\fi%

\graphicspath{{figures/}{pictures/}{images/}{./}} 

\usepackage{microtype}                 
\PassOptionsToPackage{warn}{textcomp}  
\usepackage{textcomp}                  
\usepackage{mathptmx}                  
\usepackage{times}                     
\usepackage{cite}                      
\usepackage{tabu}                      
\usepackage{booktabs}                  
\usepackage{amsfonts}
\usepackage{amsmath}
\usepackage{bm}
\usepackage{amsthm}
\usepackage{xcolor}
\usepackage{siunitx}
\usepackage{multirow}
\usepackage[ruled,linesnumbered]{algorithm2e}
\newtheorem{theorem}{Theorem}
\newtheorem{lemma}{Lemma}
\newtheorem{corollary}{Corollary}[theorem]
\DeclareMathOperator*{\argmin}{arg\,min}
\newcommand\overmat[2]{%
			\smash{\overbrace{{%
			\begin{matrix}#2\end{matrix}}}^{\text{\color{black}#1}}}\vphantom{#2}}
\newcommand\undermat[2]{%
	\smash{\underbrace{{%
				\begin{matrix}#2\end{matrix}}}_{{\color{black}#1}}}\vphantom{#2}}	
	
\SetKwInput{KwInput}{Input}
\SetKwInput{KwOutput}{Output}

\onlineid{0}

\vgtccategory{Research}

\vgtcinsertpkg

\title{An Efficient Planar Bundle Adjustment Algorithm}

\author{Lipu Zhou\thanks{e-mail: lzhou@magicleap.com} \\
	\scriptsize Magic Leap
	\and Daniel Koppel\thanks{e-mail:dkoppel@magicleap.com} \\
	\scriptsize Magic Leap
	\and Hui Ju\thanks{e-mail:hju@magicleap.com}	\\
	\scriptsize Magic Leap
	\and Frank Steinbruecker\thanks{e-mail: steinbruecker@magicleap.com } \\
	\scriptsize Magic Leap
	\and Michael Kaess \thanks{e-mail: kaess@cmu.edu} \\
	\scriptsize Carnegie Mellon University
}

\abstract{ This paper presents an efficient algorithm for the least-squares problem using the point-to-plane cost, which  aims to jointly optimize  depth sensor poses and plane parameters  for  3D  reconstruction. We call this least-squares problem \textbf{Planar Bundle Adjustment} (PBA), due to the similarity between this problem and the original Bundle Adjustment (BA) in visual reconstruction. As planes  ubiquitously exist in the man-made environment, they are generally used as landmarks in  SLAM algorithms for various depth sensors.  PBA is important to reduce drift and improve the  quality of the map.  However, directly adopting the well-established BA framework in visual reconstruction  will result in a very inefficient solution for PBA. This is because a 3D point only has one observation at a camera pose. In contrast, a depth sensor can record hundreds of points in a plane at a time, which results in a very large  nonlinear least-squares problem even for a small-scale space. The main contribution of this paper is an efficient solution for the PBA problem using the point-to-plane cost. We introduce a reduced Jacobian matrix and a reduced residual vector, and prove that they can replace the original Jacobian matrix and residual vector in the generally adopted Levenberg-Marquardt (LM) algorithm. This   significantly reduces the computational cost. 
Besides, when planes are combined with other features for 3D reconstruction, the reduced Jacobian matrix and  residual vector can also replace the corresponding parts derived from planes. Our experimental results  show that our algorithm can significantly reduce the computational time compared to the solution using the traditional BA framework. In addition, our algorithm is faster,  more accurate, and more robust to initialization errors compared to the start-of-the-art solution using the plane-to-plane cost \cite{geneva2018lips}.
} 

\CCScatlist{
  \CCScatTwelve{Bundle Adjustment}{Nonlinear Optimization}{SLAM}{Depth Sensor}{}
}

\begin{document}



\firstsection{Introduction}

\maketitle

Simultaneous localization and mapping (SLAM)   is important for  Augmented Reality (AR) systems, and many other computer vision and robotics applications. Generally, visual SLAM algorithms \cite{Klein2009ParallelTA,schops2014semi,liu2016robust,mur2017orb,liu2018ice} yield  sparse or semi-dense maps. It is difficult for visual SLAM algorithms to recover the structure of textureless  areas, such as white walls. Recently, depth sensors have become available for AR glasses and smart phones. This has made possible the building of  dense 3D maps in real time \cite{salas2014dense,kaess2015simultaneous,hsiao2017keyframe,hsiao2018dense,huang2019optimization}, which is desirable, not only for providing the necessary information for localization, but also the   information for environment understanding,  such as 3D  object classification and semantic segmentation \cite{qi2017pointnet},  required in most of today's AR applications. As planes    ubiquitously exist in  man-made scenes,  they are generally exploited as features in SLAM algorithms for  various depth sensors  \cite{pathak2010fast,lee2012indoor,trevor2012planar,salas2014dense,kaess2015simultaneous,yang2016pop,hsiao2017keyframe,zhang2017low,hsiao2018dense,kim2018linear,huang2019optimization,yang2019monocular}. In recent work \cite{li2019robust}, planes are also used in a novel Visual-Inertial Odometry (VIO) system. In addition,  planes  are important locations for an AR system to display virtual information as demonstrated in \cite{salas2014dense}. Thus planes play an  important role in AR applications. 
This paper focuses on optimizing the plane model of an environment.

It is known that bundle adjustment (BA)  is  crucial  for visual reconstruction  to generate a high quality and globally consistent 3D map. Thus BA has been extensively investigated in the literature \cite{triggs1999bundle,hartley2003multiple,lourakis2009sba,agarwal2010bundle,grisetti2010tutorial,wu2011multicore,kaess2012isam2,zach2014robust,schonberger2016structure,zhang2017distributed,liu2018ice}.
Analogously to the visual reconstruction case, we face a nonlinear least-squares problem whose goal is to obtain optimal  poses and plane parameters for depth data.   
Here we name this least-squares problem as \textbf{Planar Bundle Adjustment} (PBA), due to the similarity between  this  least-squares problem and BA.  Although  BA for  visual  reconstruction has been well studied in the literature, little research has been done on the  PBA.  Thus, this raises the demand for  studying  the  PBA problem for depth sensors.

The PBA is the problem of jointly optimizing parameters of planes and sensor poses. Although the plane is seemingly the counterpart of the point in BA for visual reconstruction, there exists a significant difference between them. In the literature, the plane-to-plane distance based on  the rigid-body transformation for plane parameters is generally employed to construct the cost function. However, this cost function may introduce bias which may result in  a suboptimal solution as described in \autoref{sect_related-work}. Thus this  paper adopts the point-to-plane distance   to construct the cost function.   A 3D point can only yield one observation for each camera pose. However, as a plane is an infinite object, one recording of a depth sensor can provide many points as a partial observation of a plane. Thus a single depth sensor  recording can generate many constraints on planes and poses. Therefore, directly adopting the original visual BA framework to the planar case will result in a large-scale nonlinear least-square  problem even for a small scene, which incurs high computational cost and memory consumption. This paper addresses this problem.

The main contribution of this paper is an efficient planar bundle adjustment algorithm.
The key point of our algorithm is to explore the  special structure of PBA. Based on the special structure of PBA, we introduce a reduced Jacobian matrix and a reduced residual vector. We prove that no matter how many points of a plane are recorded by a depth sensor, the derived Jacobian matrix and residual vector can  be replaced by the reduced  ones  in the Levenberg-Marquardt (LM) algorithm \cite{levenberg1944method,more1978levenberg}. The reduced Jacobian matrix and  residual vector have  fixed sizes that are  much smaller  than the original ones. This significantly  reduces computational cost and  memory usage. More generally, when 3D reconstruction applications exploit planes together with other features such as points \cite{taguchi2013point,hosseinzadeh2018structure,wang2019submap,grant2019efficient,yang2019tightly}, the same reduction technique can be applied to the blocks inside the  Jacobian and the residual, which correspond to the planar constraints.

\section{Related Work and Theoretical Background} \label{sect_related-work}
Planes are widely adopted as landmarks  in  SLAM algorithms for depth sensors. The related topics include plane detection,  matching,   3D registration  and joint optimization with poses and planes, \textit{etc}. 
The research on the joint optimization problem is relative small. This paper  focuses on this problem. In this section, we introduce the related work and the theoretical background for the  optimization problem. In the following description, we use italic, boldfaced lowercase and boldfaced uppercase letters to represent scalars, vectors and matrices, respectively.

\subsection{Cost Functions for the Plane Correspondence}
One requisite  step  of formulating a optimization problem is to construct the cost function. For the plane correspondence, there are two cost functions generally used in the literature, \textit{i.e.}, \textbf{plane-to-plane }and \textbf{point-to-plane}. 

\textbf{Plane-to-plane cost} is based on the rigid-body transformation for the  plane parameters. 
Specifically, suppose  a plane is represented by the Hesse normal form as $\bm{\pi} = \left[ \mathbf{n}; d\right] $, where $\mathbf{n}$ is the plane normal with  $\left\| \mathbf{n}\right\|_2 =1 $  and $d$ is the negative distance from the coordinate system origin to the plane. Assume the rotation and translation from a depth  sensor's local coordinate system to a global coordinate system are $\mathbf{R}$ and $\mathbf{t}$, respectively. Let $\bm{\pi}_s $ represent the parameters of a plane in the depth sensor coordinate system,   which are estimated from a set of points $\mathbb{P}$  in the depth sensor measurements,  and $\bm{\pi}_g$  denote  the same plane's parameters in the global coordinate system. Then the relation between $\bm{\pi}_s $  and $\bm{\pi}_g $  can be described as \cite{hartley2003multiple}
\begin{equation} \label{equ_pl2pl}
	\bm{\pi}_s = \mathbf{T}^{T}\bm{\pi}_g, \ \mathbf{T} = 
	\begin{bmatrix}
		\mathbf{R} & \mathbf{t} \\
		\mathbf{0} & 1
	\end{bmatrix}.
\end{equation}
$\bm{\pi}_g $  and $\mathbf{T}$ are  the variables that we want to estimate. We can have a general form  of the  plane-to-plane residual as below
\begin{equation} \label{equ_pl2pl_diff}
	\bm{\pi}_s  \ominus \mathbf{T}^{T}\bm{\pi}_g,
\end{equation}
where $\ominus$ is a function to measure the difference between $\bm{\pi}_s $ and $\mathbf{T}^{T}\bm{\pi}_g$. Kaess \cite{kaess2015simultaneous}  represents  planes as quaternions, and defines $\ominus$ by measuring the difference between two quaternions. Geneva \textit{et al.} \cite{geneva2018lips} introduce the closest point (CP) vector, \textit{i.e.} $d\mathbf{n}$,  to parameterize a plane.  They define $\ominus$ as the difference between two CP vectors.

The   above equation (\ref{equ_pl2pl_diff}) can be directly used to construct  cost function for jointly tuning poses and planes \cite{huang2019optimization}. But the relative plane formulation introduced by Kaess \cite{kaess2015simultaneous}  converges faster.  In the relative plane formulation, a plane is  expressed relative to  the first pose  that observes it. This formulation is adopted later in \cite{hsiao2017keyframe,hsiao2018dense} for  global joint optimization of poses and planes when loop closure occurs.  Geneva \textit{et al.} \cite{geneva2018lips} present  a similar relative plane formulation as \cite{kaess2015simultaneous}. They  introduce the CP vector to represent   the plane, which shows improved accuracy and faster convergence compared to  the plane parameterization using quaternion introduced in \cite{kaess2015simultaneous}. 

\textbf{Point-to-plane cost} is the squared  distance from a point to a plane. 
As mentioned above, $\mathbb{P}$ is the point set of a plane observed by a depth sensor. Assume $\mathbf{p}_i \in \mathbb{P}$ and $\bm{\pi}_g = \left[ \mathbf{n}_g ; d_g\right] $. The signed distance from $\mathbf{p}_i $ to $\bm{\pi}_g$ has the form as  
\begin{equation} \label{equ_p2pl}
	\delta = \mathbf{n}_g \cdot \left(\mathbf{R} \mathbf{p}_i +\mathbf{t}\right) + d_g,
\end{equation}
where $\cdot$ represents the  dot product. 
The point-to-plane cost $\delta^2$  is generally employed to calculate the pose between a local depth sensor point cloud  and a  global point cloud \cite{newcombe2011kinectfusion,zhang2017low,hsiao2017keyframe,zhou2020icra} 
rather than to  jointly optimize poses and plane parameters. This is because there generally exists a large number of points in $\mathbb{P}$, 
which leads to a very large-scale least-squares problem, if this particular cost function is used. Thus, it is  seldom adopted in the SLAM algorithm for global poses and planes joint optimization. 

\textbf{plane-to-plane vs. point-to-plane} \quad As described in \cite{hartley2003multiple}, the cost function will impact  the accuracy of the solution. 
The point-to-plane cost   is well defined. It is the squared distance from $\mathbf{p}_i $ to $\bm{\pi}_g$, which   is invariant to rigid transformations. This means the cost  is invariant to the choice of  the global coordinate system. It only depends on  estimation errors of poses and planes.  However, the plane-to-plane cost based on plane parameter transformation (\ref{equ_pl2pl}) does not have this property. 
As demonstrated in \autoref{fig_pl2pl}(a) changing the  coordinate system will vary the plane-to-plane cost. Thus  the accuracy of the result of minimizing the plane-to-plane cost may depend on  the particular choice of  coordinate systems, which introduces uncertainty to the solution. 

Furthermore, the plane-to-plane cost may introduce bias. For example, given the ground truth of  the two poses and $\bm{\pi}_g$, the two local point clouds in  \autoref{fig_pl2pl}(b)   generate different plane-to-plane costs, although they yield the same point-to-plane costs. Plane $\bm{\pi}_{s}^{1}$  has smaller cost than $\bm{\pi}_{s}^{2}$ merely because of  the certain choice of the global coordinate system rather than the second pose error is larger. However, as plane $\bm{\pi}_{s}^{1}$  generates smaller cost than $\bm{\pi}_{s}^{2}$,  
this may cause  the optimization algorithm to change the second pose of the sensor to push  plane $\bm{\pi}_{s}^{2}$   toward plane $\bm{\pi}_{s}^{1}$  to minimize the plane-to-plane cost, which is prone to increase errors.

Our experimental results show that minimizing the point-to-plane cost results in a more accurate solution than  minimizing  the plane-to-plane cost. In addition, our results show that the point-to-plane cost is more robust to initialization errors.  However, using the point-to-plane cost yields a much larger scale least-squares problem. This paper presents an efficient solution to address this problem.

\begin{figure}[tb]
	\centering 
	\includegraphics[width=\columnwidth]{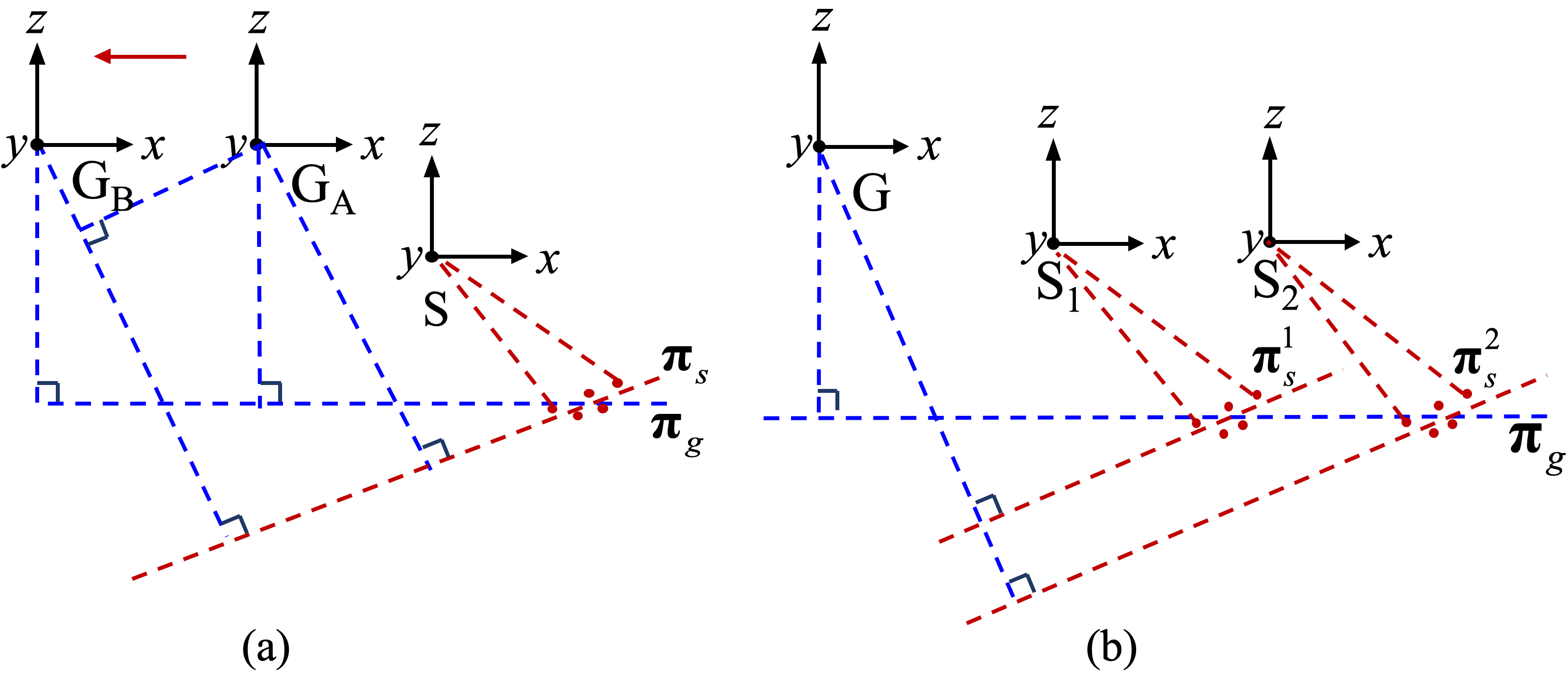}
	\caption{A schematic illustrating the potential problems of the plane-to-plane
		cost. (a) Plane-to-plane cost is not invariant to rigid transformations. When the coordinate system shifts from $\text{G}_A$ to $\text{G}_B$, the plane-to-plane cost for $\bm{\pi}_{s}\leftrightarrow \bm{\pi}_{g}$ changes. However, the point-to-plane cost is invariant.
		(b) Plane-to-plane cost may introduce bias. The two point clouds acquired from the global plane $\bm{\pi}_{g}$ at two poses ($\text{S}_1$ and $\text{S}_2$) have the same noise level. Their sums of squared distances to $\bm{\pi}_{g}$ are the same. Thus they have the same point-to-plane cost. However, the plane-to-plane costs for $\bm{\pi}_{s}^1 \leftrightarrow \bm{\pi}_{g}$  and $\bm{\pi}_{s}^2 \leftrightarrow \bm{\pi}_{g}$ are different, as  $\bm{\pi}_{s}^{1}$ and $\bm{\pi}_{s}^{2}$ have different parameters. If we consider the CP vector difference \cite{geneva2018lips}, $\bm{\pi}_{s}^1 $ will yield a smaller cost. This may cause the optimization algorithm to change the second pose  to push $\bm{\pi}_{s}^2 $ toward $\bm{\pi}_{s}^1 $ to reduce the plane-to-plane cost. This bias may result in less accurate results.}
	\label{fig_pl2pl} 
\end{figure}

\subsection{The Bundle Adjustment Problem}
3D reconstruction is a fundamental problem  with a large number of applications. BA is a crucial step for 3D reconstruction.  Let us first consider the  well-established BA framework for visual reconstruction. 

The  BA framework for  visual reconstruction is essentially a  nonlinear least-squares problem with sparse structure \cite{triggs1999bundle,wu2011multicore,agarwal2010bundle,lourakis2009sba,liu2018ice}. The  LM algorithm \cite{levenberg1944method,more1978levenberg} is generally adopted to solve this problem.  Let us begin with  a brief introduction of the LM algorithm for a general least-squares problem. 

Suppose that we have an  $n$-dimensional measurements $\mathbf{y}=\left[ y_1;y_2; \dots; y_n\right] $  modeled by function $ \mathbf{f}\left( \mathbf{x}\right) = \left[ f_1\left(\mathbf{x} \right); f_2\left(\mathbf{x}\right); \dots; f_n\left(\mathbf{x} \right) \right] $, where $\mathbf{x} \in \mathbb{R}^{m}$ is an $m$-dimensional model parameter vector which we seeks to estimate. The least-squares problem is to find the optimal $\hat{\mathbf{x}}$ which minimizes the sum of squared errors of $ \bm{\delta}\left( \mathbf{x}\right) =  \mathbf{f}\left( \mathbf{x}\right) - \mathbf{y}$, \textit{i.e.}, 
\begin{equation} \label{equ_lm}
\hat{\mathbf{x}} = \argmin_{\mathbf{x}}\frac{1}{2}\left\| \bm{\delta}\left( \mathbf{x}\right) \right\|_2^2.
\end{equation}
We use $\mathbf{J}\left( \mathbf{x}\right)$ to  represent  the Jacobian matrix of $\mathbf{f}\left( \mathbf{x}\right) $ at $\mathbf{x}$, 
whose $i$th row  and $j$th column element is $J_{ij}\left(\mathbf{x}\right) = \frac{\partial{f_i\left(\mathbf{x} \right) }}{\partial{x_j}}$, where $x_j$ is the $j$th element of $\mathbf{x}$.  To keep the expression simple in the following description, we only use the name of the function and do not explicitly include the variable of the function unless necessary.  

Given an initial estimation of the model parameters, the  LM algorithm iteratively refines the solution. At each iteration, the LM algorithm calculates the step $\bm{\xi}$ by solving the following linear system
\begin{equation} \label{equ_lm_update}
	\left( \mathbf{J}^{T} \mathbf{J}+\lambda\mathbf{I}\right) \bm{\xi} =  -\mathbf{J}^{T} \bm{\delta},
\end{equation}
where $\mathbf{I}$ denotes an $m \times m$ identity matrix and  $\lambda$ is a scalar that is adjusted at each iteration to ensure that $\bm{\xi}$ leads to a reduced cost.   After we solve the linear system (\ref{equ_lm_update}), $\mathbf{x}$ is updated by $\mathbf{x} \leftarrow \mathbf{x} + \bm{\xi}$. 

As mentioned above, the BA for visual reconstruction is a nonlinear least-squares problem. It  refines  camera poses $\mathbf{x}_c$ and 3D point coordinates $\mathbf{x}_p$  to minimize the sum of squared re-projection errors. Typically, the model parameters are  organized as $\mathbf{x} = \left[ \mathbf{x}_c; \mathbf{x}_p\right] $. Thus, the Jacobian matrix can be divided as $\mathbf{J} = \left[ \mathbf{J}_c, \mathbf{J}_p\right] $.   According to this structure, the equation system (\ref{equ_lm_update}) can be  rewritten as
\begin{equation} \label{equ_lm_block}
	\begin{bmatrix}
		\mathbf{A} & \mathbf{W}  \\ 
		\mathbf{W}^{T} & \mathbf{B} 
	\end{bmatrix}
	\begin{bmatrix}
		\bm{\xi}_c \\
		\bm{\xi}_p
	\end{bmatrix} =
	-\begin{bmatrix}
		\mathbf{J}_c^T \bm{\delta} \\
		\mathbf{J}_p^T \bm{\delta}
	\end{bmatrix},
\end{equation}
where $\mathbf{A} = \mathbf{J}_c^T\mathbf{J}_c + \lambda_c\mathbf{I}$, $\mathbf{B} = \mathbf{J}_p^T\mathbf{J}_p +  \lambda_p\mathbf{I}$ and  $\mathbf{W} = \mathbf{J}_c^T\mathbf{J}_p $. $\mathbf{A}$ and $\mathbf{B}$ are usually  block diagonal matrices. The Schur complement trick is generally adopted to solve this sparse linear system \cite{triggs1999bundle}.

The planar BA problem yields a similar structure as (\ref{equ_lm_block}).   But we cannot directly adopt the above method. This is because one plane can generate  many observations in one depth sensor recording, 
which can make the size of the resulting nonlinear least-squares problem prohibitively large. It would be time-consuming  to just  compute the Jacobian matrix $\mathbf{J}$ and the residual $\bm{\delta}$, not to mention to construct and  solve the linear system (\ref{equ_lm_block}). 
This paper shows that $\mathbf{J}$  of  PBA has a special structure, which can be used to significantly reduce the computational cost.

\section{Planar Bundle Adjustment}
In this section, we elaborate our solution for the PBA problem. We begin with the formulation of   the PBA problem. We then present our solution for this problem.



\subsection{Problem Formulation}

\begin{figure}[tb]
	\centering 
	\includegraphics[width=\columnwidth]{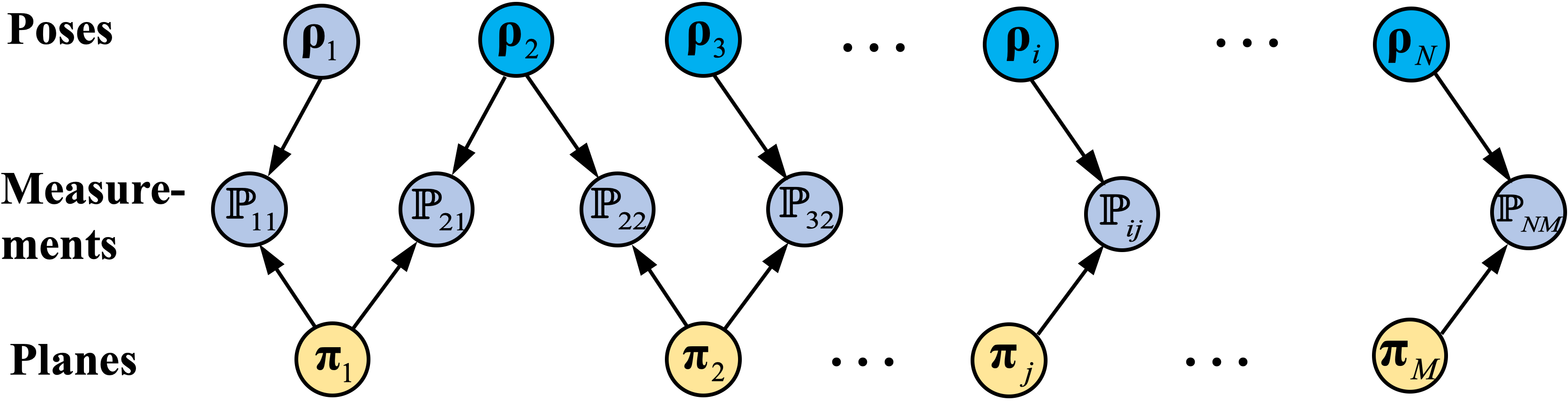}
	\caption{Bayes Net of PBA with $N$ poses and $M$ planes. The first pose $\bm{\rho}_{1}$ is fixed during the optimization. $\mathbb{P}_{ij}$ is the set of measurements of the $j$th plane $\bm{\pi}_{j}$ recorded at the $i$th pose $\bm{\rho}_{i}$. PBA is the problem of jointly refining $\bm{\rho}_{i}$ ($i \neq 1$) and $\bm{\pi}_{j}$.} 
	\label{fig_pose_graph}
\end{figure}

\begin{figure}[tb]
	\centering 
	\includegraphics[width=0.45\columnwidth]{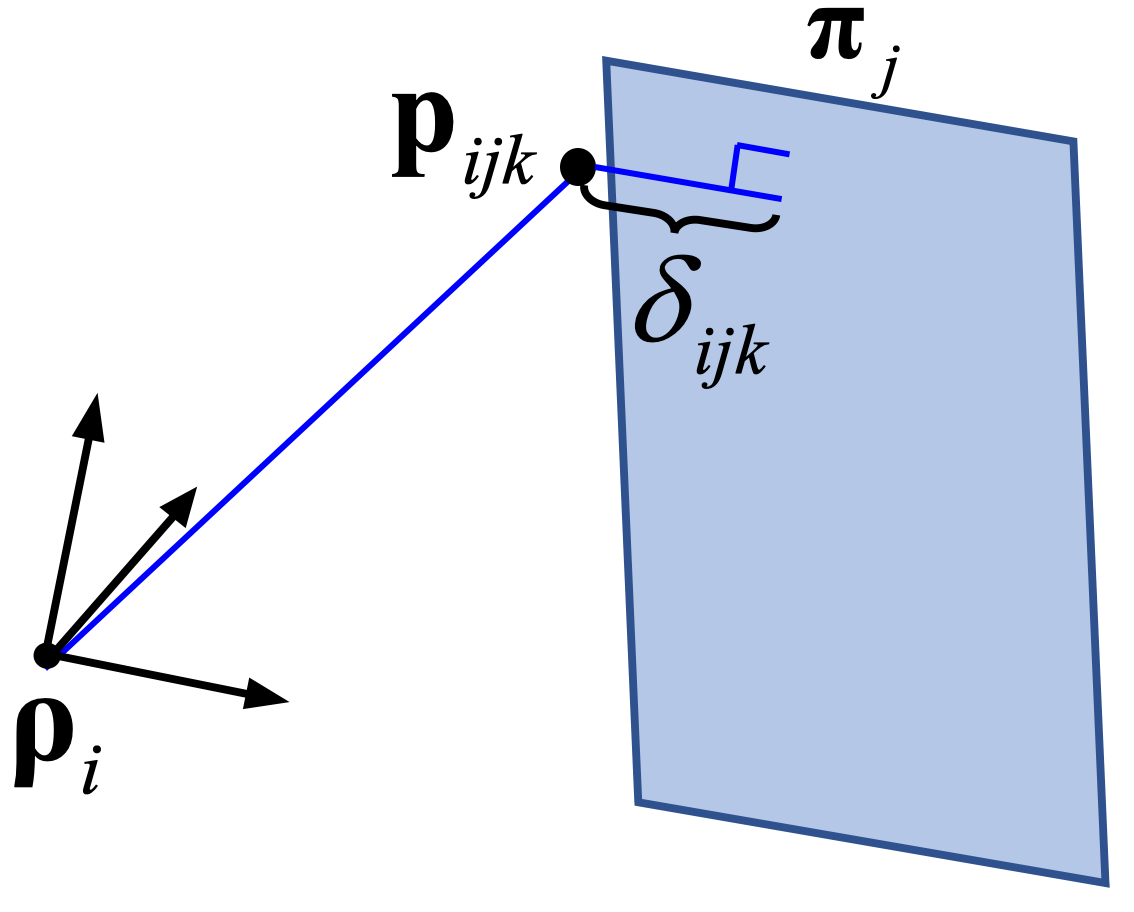}
	\caption{The geometric entities involved in PBA. 
		$\mathbf{p}_{ijk}$ is the $k$th point of the point set $\mathbb{P}_{ij}$ that  are captured from the $j$th plane $\bm{\pi}_{j}$   at the $i$th pose $\bm{\rho}_{i}$. $\delta_{ijk}$ is the signed distance from  $\mathbf{p}_{ijk}$ to   $\bm{\pi}_{j}$  defined in  (\ref{equ_res_pi}).} 
	\label{fig_delta}
\end{figure}

\autoref{fig_pose_graph} presents the Bayes Net of a PBA problem. We assume that there are $M$ planes and  $N$ sensor poses. Denote the rotation and translation of the $i$th pose as $\mathbf{R}_i \in SO_3$ and $\mathbf{t}_i \in \mathbb{R}^3$, respectively.  Suppose the $j$th plane has the parameters  $\bm{\pi}_j = \left[ \mathbf{n}_{j}; d_j\right]$. 
The measurements of the $j$th plane at the $i$th pose are a set of $K_{ij}$ points defined as:
\begin{equation}
	\mathbb{P}_{ij} = \left\lbrace \mathbf{p}_{ijk}\right\rbrace _{k=1}^{K_{ij}}
\end{equation}
Each $\mathbf{p}_{ijk} \in \mathbb{P}_{ij}$ provides one constraint on the $i$th pose and $j$th plane, which is demonstrated in \autoref{fig_delta}.  The residual   $\delta_{ijk}$ is the signed distance from $\mathbf{p}_{ijk}$ to plane $\bm{\pi}_j$ which can be  written as
\begin{equation} \label{equ_res_pi}
	\delta_{ijk} = \mathbf{n}_{j} \cdot \left( \mathbf{R}_i\mathbf{p}_{ijk} + \mathbf{t}_i\right)  + d_j,
\end{equation}

Unlike $\mathbf{t}_i$, the rotation $\mathbf{R}_i$ and plane parameters $\bm{\pi}_j $ have extra constraints. The parameterization of both entities has been well-studied. For instance, $\mathbf{R}_i$ can be parameterized by a quaternion, angle-axis or Euler angles. The plane parameters $\bm{\pi}_j$ can be represented by homogeneous coordinates \cite{hartley2003multiple},  closest point parameterization \cite{yang2019tightly} or the minimal parameterization based on quaternion introduced in \cite{kaess2015simultaneous}. Our algorithm does not depend on  a special  parameterization. We define $\bm{\theta}_i \mapsto \mathbf{R}\left( \bm{\theta}_i\right) $ and $\bm{\omega}_j \mapsto \bm{\pi}\left( \bm{\omega}_j\right)$ to represent arbitrary  parameterization for rotation and plane parameters. $\bm{\theta}_i$ and  $\mathbf{t}_i$ are related to the sensor pose. According to the convention of   visual BA introduced above, we  combine them as $\bm{\rho}_i= \left[\bm{\theta}_i; \mathbf{t}_i \right]$.    Generally,  $\bm{\rho}_i$ has $6$ or $7$ unknowns,  depending on the parameterization of the rotation matrix (6 for  the cases of any minimal representations of  the rotation and 7 for the quaternion). $\bm{\omega}_{j}$  generally has 3 or  4 unknowns (3 for minimal representations of a  plane \cite{kaess2015simultaneous,yang2019tightly} and 4 for the  homogeneous coordinates of a plane\cite{hartley2003multiple})

Using the above notations, $\delta_{ijk}$ is a function of $\bm{\rho}_i$ and $ \bm{\omega}_j$. The PBA is the problem of jointly refining all $\bm{\rho}_i$ ($i \neq 1$) and $\bm{\omega}_j$ by minimizing the following nonlinear least-squares problem
\begin{equation} \label{equ_cost}
	\min_{\substack{\bm{\rho}_i, \bm{\omega}_ j  \\ i \neq 1}} \ \sum_{i}\sum_{j} \sum_{k}\delta_{ijk} ^2\left( \bm{\rho}_i,\bm{\omega}_j\right).
\end{equation}   
Here the first pose $\bm{\rho}_1$ is fixed  during the optimization to anchor the coordinate system rigidly.
\subsection{Structure of the  Jacobian Matrix }
One crucial step of the LM algorithm is to calculate the Jacobian matrix. The Jacobian matrix of planar BA has a special structure. The observation of the $j$th plane at the $i$th pose is a point set $\mathbb{P}_{ij}$.   Let us consider the Jacobian matrix $\mathbf{J}_{ij}$ derived from   $\mathbb{P}_{ij}$. The whole Jacobian matrix $\mathbf{J}$ is the stack of all $\mathbf{J}_{ij}$. Assume there are  $K_{ij}$ points in $\mathbb{P}_{ij}$.  $\mathbf{J}_{ij}$ ($ i \neq 1$) has the following form 

\begin{equation} \label{equ_Jij}
\mathbf{J}_{ij} =\underbrace{ \left[ \begin{matrix}
	\mathbf{0} & \cdots & \overmat{\textit{i}\textbf{th pose}}{\frac{\partial \delta_{ij1} }{\partial {\bm{\rho}}_{i}}} & \cdots  & \mathbf{0} \\
	\mathbf{0} & \cdots & 	\frac{\partial \delta_{ij2} }{\partial {\bm{\rho}}_{i}} &\cdots  & \mathbf{0} \\
	\vdots & \ddots & 	\vdots & \ddots & \vdots \\
	\mathbf{0}& \cdots & 	\frac{\partial \delta_{ij{K_{ij}}} }{\partial {\bm{\rho}}_{i}} & \cdots &  \mathbf{0} \\
\end{matrix} \right. }_{N-1 \ \text{\textbf{pose}}} 
\underbrace{
	\left.
	\begin{matrix}
		\quad \cdots &  \overmat{\textit{j}\textbf{th plane}}{\frac{\partial \delta_{ij1} }{\partial {\bm{\omega}}_{j}}}  & \cdots & \mathbf{0} \\
		\quad \cdots & \frac{\partial \delta_{ij2} }{\partial {\bm{\omega}}_{j}} & \dots & \mathbf{0} \\
		\quad \ddots& \vdots & \ddots &\vdots \\
		\quad \cdots & \frac{\partial \delta_{ij{K_{ij}}} }{\partial {\bm{\omega}}_{j}} & \cdots & \mathbf{0} \\
	\end{matrix}
	\right] }_{M \ \text{\textbf{plane}}}
\end{equation}

To calculate 	$\frac{\partial \delta_{ijk} }{\partial {\bm{\rho}}_{i}}$ and $\frac{\partial \delta_{ijk} }{\partial {\bm{\omega}}_{j}}$, we begin by considering the form of $\delta_{ijk}$ in  (\ref{equ_res_pi}).  Let us define
\begin{equation} \label{equ_define}
	\mathbf{R}_{i} = \begin{bmatrix}
	R_i^{1,1} & R_i^{1,2} & R_i^{1,3} \\ 
	R_i^{2,1} & R_i^{2,2} & R_i^{2,3} \\ 
	R_i^{3,1} & R_i^{3,2} & R_i^{3,3} \\ 
	\end{bmatrix},
	\mathbf{t}_{i} = \begin{bmatrix}
		t_{i}^{1} \\
		t_{i}^{2} \\
		t_{i}^{3} 
	\end{bmatrix},
	\mathbf{n}_{j} = \begin{bmatrix}
		n_{j}^{1} \\
		n_{j}^{2} \\
		n_{j}^{3} 
	\end{bmatrix},
	\mathbf{p}_{ijk} = \begin{bmatrix}
		x_{ijk} \\
		y_{ijk} \\
		z_{ijk} 
	\end{bmatrix}.
\end{equation}
Note that the elements of $\mathbf{R}_i$  defined above are functions of $\bm{\theta}_i$, and also that $d_j$ and the elements of $\mathbf{n}_j$ are functions of $\bm{\omega}_j$. Substituting (\ref{equ_define}) into (\ref{equ_res_pi}) and expanding it, we have
\begin{equation}
	\begin{split} \label{equ_res_exp}
		\delta_{ijk} =  & x_{ijk}R_{i}^{1,1}n_j^1 +  y_{ijk}R_{i}^{1,2}n_j^1 + z_{ijk}R_{i}^{1,3}n_j^1   + \\
							   & x_{ijk}R_{i}^{2,1}n_j^2 + y_{ijk}R_{i}^{2,2}n_j^2 + z_{ijk}R_{i}^{2,3}n_j^2  + \\
							   & x_{ijk}R_{i}^{3,1}n_j^3 + y_{ijk}R_{i}^{3,2}n_j^3 + z_{ijk}R_{i}^{3,3}n_j^3   + \\
							   & n_{j}^{1}t_{i}^{1} + n_{j}^{2}t_{i}^{2} + n_{j}^{3}t_{i}^{3} + d_j.
	\end{split}
\end{equation}
We can rewrite (\ref{equ_res_exp}) as
\begin{equation} \label{equ_delta_exp}
\delta_{ijk}   =  \mathbf{c}_{ijk} \cdot \bm{\nu}_{ij}, \\
\end{equation}
where $\mathbf{c}_{ijk}$ and $\bm{\nu}_{ij}$ are 13-dimensional vectors as
\begin{equation} \label{equ_delta_definition}
	\begin{split}
		\mathbf{c}_{ijk}  =  & \begin{bmatrix}
			x_{ijk}, \, y_{ijk} ,\, z_{ijk}, \, x_{ijk}, \, y_{ijk} ,\, z_{ijk}, \, x_{ijk}, \, y_{ijk} ,\, z_{ijk}, 1, 1, 1, 1
		\end{bmatrix}^{T}, \\
		\bm{\nu}_{ij}  = &
			[R_{i}^{1,1}n_j^1, \,  R_{i}^{1,2}n_j^1, \, R_{i}^{1,3}n_j^1, \, R_{i}^{2,1}n_j^2, \,  R_{i}^{2,2}n_j^2, \, R_{i}^{2,3}n_j^2, \, R_{i}^{3,1}n_j^3,  \\
			&  R_{i}^{3,2}n_j^3, \, R_{i}^{3,3}n_j^3, \, n_{j}^{1}t_{i}^{1}, \, n_{j}^{2}t_{i}^{2}, \, n_{j}^{3}t_{i}^{3}, \, d_j ]^{T}.
	\end{split}
\end{equation}
The elements in $\mathbf{c}_{ijk}$ are from the observation $\mathbf{p}_{ijk}$ or 1. They are constants. On the other hand,  the elements in $\bm{\nu}_{ij}$ are functions of $\bm{\rho}_i$ and $\bm{\omega}_j$. They are related to the unknowns which we want to estimate. 

Let us calculate the partial derivative of $\delta_{ijk}$. Assume that $\bm{\rho}_{i}$ has  $n_{\bm{\rho}}$ unknowns  and $\bm{\omega}_{j}$ has $n_{\bm{\omega}}$ unknowns. 
We define 
\begin{equation}
	\bm{\zeta}_{ij}  = \begin{bmatrix}
		\bm{\rho}_{i} \\
		\bm{\omega}_{j}
	\end{bmatrix}
	\setlength\arraycolsep{1pt}
	\begin{matrix}
		\left. \right\rbrace & n_{\bm{\rho}}  & \text{unknowns}\\
		\left. \right\rbrace & n_{\bm{\omega}} & \text{unknowns}
	\end{matrix}
\end{equation}
Suppose $\zeta_{ij}^d$ is the $d$th element of $\bm{\zeta}_{ij}$. 
According to (\ref{equ_delta_exp}), the partial derivative of $\delta_{ijk}$ with respect to  $\zeta_{ij}^{d}$ has the following form
\begin{equation} \label{equ_partial_der}
\frac{\partial \delta_{ijk} }{\partial \zeta_{ij}^{d}} =  \frac{\partial \mathbf{c}_{ijk} \cdot \bm{\nu}_{ij}}{\partial \zeta_{ij}^{d}}= \mathbf{c}_{ijk} \cdot \frac{\partial \bm{\nu}_{ij}}{\partial \zeta_{ij}^{d}},
\end{equation}
where $ \frac{\partial \bm{\nu}_{ij}}{\partial \zeta_{ij}^{d}}$ is a 13-dimensional vector whose elements are the partial derivatives  of  the elements of $\bm{\nu}_{ij}$ with respect to  $\zeta_{ij}^d$. 

Then we consider $\frac{\partial \delta_{ijk} }{\partial \bm{\rho}_{i}}$  and $\frac{\partial \delta_{ijk} }{\partial \bm{\omega}_j}$. 
According to (\ref{equ_partial_der}), $\frac{\partial \delta_{ijk} }{\partial \bm{\rho}_{i}}$ has the following form
\begin{equation} \label{equ_d_rho}
	\begin{split}
		\frac{\partial \delta_{ijk} }{\partial {\bm{\rho}}_{i}} & = \left[ 
		\frac{\partial \delta_{ijk} }{\partial \zeta_{ij}^{1}}, \, \dots, \, \frac{\partial \delta_{ijk} }{\partial \zeta_{ij}^{n_{\bm{\rho}}}}
		\right] \\
		& = \left[ \mathbf{c}_{ijk} \cdot \frac{\partial \bm{\nu}_{ij}}{\partial \zeta_{ij}^{1}}, \, \dots, \,   \mathbf{c}_{ijk} \cdot \frac{\partial \bm{\nu}_{ij}}{\partial \zeta_{ij}^{n_{\bm{\rho}}}} 
		\right] \\
		& = \mathbf{c}_{ijk}^{T}\underbrace{\left[  \frac{\partial \bm{\nu}_{ij}}{\partial \zeta_{ij}^{1}}, \, \dots, \,  \frac{\partial \bm{\nu}_{ij}}{\partial \zeta_{ij}^{n_{\bm{\rho}}}} 
		\right]}_{\bm{V}^{{\bm{\rho}}_{i}}_{ij}}
		 = \mathbf{c}_{ijk}^{T}\bm{V}^{{\bm{\rho}}_{i}}_{ij}.
	\end{split}
\end{equation}
$\bm{V}^{{\bm{\rho}}_{i}}_{ij}$ is generally a $13 \times 6$ or  $13 \times 7$ matrix ($13 \times 6$ for  minimal representations of the rotation matrix and $13 \times 7$  for quaternion). 

Similarly, we can calculate $\frac{\partial \delta_{ijk} }{\partial \bm{\omega}_{j}} $  as 
\begin{equation} \label{equ_d_omega}
	\begin{split}
		\frac{\partial \delta_{ijk} }{\partial \bm{\omega}_{j}} 
		& = \mathbf{c}_{ijk}^{T}\underbrace{\left[  \frac{\partial \bm{\nu}_{ij}}{\partial \zeta_{ij}^{n_{\bm{\rho}}+1}}, \, \dots, \,  \frac{\partial \bm{\nu}_{ij}}{\partial \zeta_{ij}^{n_{\bm{\rho}}+n_{\bm{\omega}}}} 
			\right]}_{\bm{V}^{{\bm{\omega}}_{i}}_{ij}}
		=  \mathbf{c}_{ijk}^{T}\bm{V}^{{\bm{\omega}}_{j}}_{ij}
	\end{split}
\end{equation}
Typically, $\bm{V}^{{\bm{\omega}}_{j}}_{ij} $ is a $13 \times 3$ or  $13 \times 4$ matrix ($13 \times 3$ for minimal representations of a plane \cite{kaess2015simultaneous,yang2019tightly} and $13 \times 4$ for homogeneous coordinates  of a plane\cite{hartley2003multiple}).

Now we consider the form of $\mathbf{J}_{ij}$. Let us define
\begin{equation} \label{equ_C_ij}
\mathbf{C}_{ij} =  \begin{bmatrix}
\mathbf{c}_{ij1}^{T} \\
\mathbf{c}_{ij2}^{T} \\
\vdots \\
\mathbf{c}_{ijK_{ij}}^{T}
\end{bmatrix},
\end{equation}
where  the $k$th row $\mathbf{c}_{ijk}^{T} $ is defined in (\ref{equ_delta_definition}).  $\mathbf{C}_{ij}$ is a matrix of size $K_{ij} \times 13$. Substituting (\ref{equ_d_rho}) and (\ref{equ_d_omega}) into (\ref{equ_Jij}) and using the definition of $\mathbf{C}_{ij}$ in (\ref{equ_C_ij}), we have
\begin{equation} \label{equ_Jij_block}
	\begin{split}
		\mathbf{J}_{ij} = & \begin{bmatrix}
			\mathbf{0} & \cdots & \mathbf{c}_{ij1}^{T}\mathbf{V}^{\bm{\rho}_{i}}_{ij} & \cdots  & \mathbf{0} & \cdots &   \mathbf{c}_{ij1}^{T}\mathbf{V}^{\bm{\omega}_{j}}_{ij}& \cdots & \mathbf{0} \\
			\mathbf{0} & \cdots & \mathbf{c}_{ij2}^{T}\mathbf{V}^{\bm{\rho}_{i}}_{ij} & \cdots  & \mathbf{0} & \cdots &   \mathbf{c}_{ij2}^{T}\mathbf{V}^{\bm{\omega}_{j}}_{ij} & \cdots & \mathbf{0} \\
			\vdots & \ddots & 	\vdots & \ddots & \vdots & \ddots& \vdots & \ddots &\vdots\\
			\mathbf{0} & \cdots &  
			\smash{\underbrace{\mathbf{c}_{ijK_{ij}}^{T}\mathbf{V}^{\bm{\rho}_{i}}_{ij} }_{\mathbf{C}_{ij}\mathbf{V}^{\bm{\rho}_{i}}_{ij} } }& \cdots   & \mathbf{0} & 
			\cdots & \smash{\underbrace{\mathbf{c}_{ijK_{ij}}^{T}\mathbf{V}^{\bm{\omega}_{j}}_{ij} }_{\mathbf{C}_{ij}\mathbf{V}^{\bm{\omega}_{j}}_{ij} }}& \cdots & \mathbf{0} \\
		\end{bmatrix}\\	
		\\
		\\
		=&\begin{bmatrix}
			\mathbf{0} & \cdots & \mathbf{C}_{ij}\bm{V}^{{\bm{\rho}}_{i}}_{ij} &  \cdots  & \mathbf{0} & \cdots & \mathbf{C}_{ij}\bm{V}^{{\bm{\omega}}_{j}}_{ij} & \cdots  & \mathbf{0}
		\end{bmatrix}.
	\end{split}
\end{equation}

\textbf{Jacobian Matrix  for the First Pose} \quad
Suppose $\mathbb{P}_{1j}$ is the set of measurements from the $j$th plane $\bm{\omega}_{j}$ at the first pose  $\bm{\rho}_1$. As we fix $\bm{\rho}_1$  during the optimization, the Jacobian matrix $\mathbf{J}_{1j} $ derived from $\mathbb{P}_{1j}$  has a special form as
\begin{equation} \label{equ_J1j}
\mathbf{J}_{1j} =\begin{bmatrix}
\mathbf{0} & \cdots & \mathbf{0} &  \cdots  & \mathbf{0} & \cdots & \mathbf{C}_{1j}\bm{V}^{{\bm{\omega}}_{j}}_{1j} & \cdots  & \mathbf{0}
\end{bmatrix}
\end{equation}
\subsection{Factorization of $\mathbf{C}_{ij}$  }
 $\mathbf{C}_{ij}$ has a special structure. It has duplicated columns, according to the structure of $\mathbf{c}_{ijk}$ defined in (\ref{equ_delta_definition}).  We have the following lemma about $\mathbf{C}_{ij}$ :

\begin{lemma} \label{lemma_C_ij}
	$\mathbf{C}_{ij}$ can be written in the following form 
	\begin{equation} \label{equ_qr}
	\mathbf{C}_{ij} = \mathbf{Q}_{ij}\mathbf{M}_{ij},
	\end{equation}
	where  $\mathbf{M}_{ij}$  has the size $4 \times 13$ and $\mathbf{Q}_{ij}^{T}\mathbf{Q}_{ij} = \mathbf{I}_{4}$, where $\mathbf{I}_{4}$ is the $4 \times 4 $ identity matrix.
\end{lemma}
\begin{proof}
	As shown in the definition of $\mathbf{c}_{ijk}$ in (\ref{equ_delta_definition}), $x_{ijk}$, $y_{ijk}$, $z_{ijk}$  and 1 are duplicated several times  to form $\mathbf{c}_{ijk}$. Therefore, there are only 4 unique columns  among the 13 columns of $\mathbf{C}_{ij}$, which contains the constant 1 and the  $x$, $y$, $z$ coordinates of points within $\mathbb{P}_{ij}$. We denote them as
	\begin{equation}  \label{equ_qr_E_ij}	
	\begin{split}\mathbf{E}_{ij} = &
		\begin{bmatrix}
			x_{ij1} & y_{ij1} & z_{ij1} & 1 \\
			x_{ij2} & y_{ij2} & z_{ij2} & 1 \\
			\vdots  & \vdots  & \vdots  & \vdots \\
			\undermat{\mathbf{x}_{ij}}{x_{ijK_{ij}}} & \undermat{\mathbf{y}_{ij}}{y_{ijK_{ij}}} & \undermat{\mathbf{z}_{ij}}{z_{ijK_{ij}} }& \undermat{\mathbf{1}}{1}
		\end{bmatrix} 
		= \begin{bmatrix}
			\mathbf{x}_{ij} & \mathbf{y}_{ij} & \mathbf{z}_{ij} & \mathbf{1}
		\end{bmatrix}.
	\end{split} 
	\end{equation}
	 The 13 columns in $\mathbf{C}_{ij}$ are  simply copies of the 4 columns in $\mathbf{E}_{ij}$. 
	  Let us define the   thin QR decomposition \cite{golub2012matrix}   of $\mathbf{E}_{ij}$  as  
	 \begin{equation} \label{equ_E_qr}
	 \mathbf{E}_{ij} =  \mathbf{Q}_{ij}\mathbf{U}_{ij}\\
	 \end{equation}
	 where  $\mathbf{Q}_{ij}^{T}\mathbf{Q}_{ij} = \mathbf{I}_{4}$ and $\mathbf{U}_{ij}$ is an upper triangular matrix. $\mathbf{Q}_{ij}$ is of  size $K_{ij} \times 4$  and $\mathbf{U}_{ij}$ is of  size $4 \times 4$. Here we use the  thin QR decomposition, since  the number of points $K_{ij}$ is generally much larger than 4.  The thin QR decomposition can reduce computational time. We partition $\mathbf{U}_{ij}$ into its columns as
	 \begin{equation} \label{equ_U_ij}
	 	\mathbf{U}_{ij}  = \begin{bmatrix}
	 		\mathbf{u}_{ij}^{1} & \mathbf{u}_{ij}^{2} & \mathbf{u}_{ij}^{3} & \mathbf{u}_{ij}^{4}  
	 	\end{bmatrix}.
	 \end{equation}
	 Substituting (\ref{equ_U_ij}) into (\ref{equ_qr_E_ij}), we have
	 \begin{equation} \label{equ_E=QU}
	 	\begin{split}
			\mathbf{E}_{ij} 
			& = \mathbf{Q}_{ij}\begin{bmatrix}
			\mathbf{u}_{ij}^{1} & \mathbf{u}_{ij}^{2} & \mathbf{u}_{ij}^{3} & \mathbf{u}_{ij}^{4}  
			\end{bmatrix} \\
	 	\end{split}
	 \end{equation}
	 Comparing (\ref{equ_qr_E_ij}) and (\ref{equ_E=QU}), we get
	 \begin{equation} \label{equ_xyz1}
	 	\begin{split}
	 		\mathbf{x}_{ij} &= \mathbf{Q}_{ij}\mathbf{u}_{ij}^{1}, \quad
	 		\mathbf{y}_{ij} = \mathbf{Q}_{ij}\mathbf{u}_{ij}^{2}, \\
	 		\mathbf{z}_{ij} &= \mathbf{Q}_{ij}\mathbf{u}_{ij}^{3}, \quad
	 		\mathbf{1} = \mathbf{Q}_{ij}\mathbf{u}_{ij}^{4}
	 	\end{split}
	 \end{equation}
	As  the columns of $\mathbf{C}_{ij}$ are copies of the columns of $\mathbf{E}_{ij}$, according to the form of $\mathbf{c}_{ijk}$ in (\ref{equ_delta_definition}) and the definition of $\mathbf{E}_{ij}$ in (\ref{equ_qr_E_ij}) , $\mathbf{C}_{ij}$ can be written as
	\begin{equation} \label{equ_C_ij_col}
			\mathbf{C}_{ij}  =\begin{bmatrix}
			\mathbf{x}_{ij} , \  \mathbf{y}_{ij}, \ \mathbf{z}_{ij}, \ \mathbf{x}_{ij}, \  \mathbf{y}_{ij}, \ \mathbf{z}_{ij}, \  \mathbf{x}_{ij}, \  \mathbf{y}_{ij}, \ \mathbf{z}_{ij}, \  \mathbf{1}, \     \mathbf{1}, \  \mathbf{1}, \  \mathbf{1}
			\end{bmatrix}
	\end{equation}
	Substituting (\ref{equ_xyz1}) into (\ref{equ_C_ij_col}), we finally have 
	\begin{equation} \label{equ_C_ij_QijMij}
		\begin{split}
			\mathbf{C}_{ij} &  =\mathbf{Q}_{ij}\underbrace{\begin{bmatrix}
				\mathbf{u}_{ij}^{1} , \  \mathbf{u}_{ij}^{2}, \ \mathbf{u}_{ij}^{3}, \ \mathbf{u}_{ij}^{1}, \  \mathbf{u}_{ij}^{2}, \ \mathbf{u}_{ij}^{3}, \  \mathbf{u}_{ij}^{1}, \  \mathbf{u}_{ij}^{2}, \ \mathbf{u}_{ij}^{3}, \  \mathbf{u}_{ij}^{4}, \  \mathbf{u}_{ij}^{4}, \  \mathbf{u}_{ij}^{4}, \  \mathbf{u}_{ij}^{4}
				\end{bmatrix}}_{\mathbf{M}_{ij}} \\
			&= \mathbf{Q}_{ij}\mathbf{M}_{ij}
		\end{split}
	\end{equation}
\end{proof}

The factorization of $\mathbf{C}_{ij}$ can be used to  significantly reduce the computational cost as described below.

\subsection{Reduced Jacobian Matrix }


According to Lemma \autoref{lemma_C_ij}, $\mathbf{C}_{ij}$ can be factorized as $\mathbf{C}_{ij} = \mathbf{Q}_{ij}\mathbf{M}_{ij}$.   We  define the \textbf{reduced  Jacobian matrix} $\mathbf{J}_{ij}^{r}$  of $\mathbf{J}_{ij}$ as 
\begin{equation} \label{equ_Jr_ij}
	\mathbf{J}_{ij}^{r} =\begin{cases} 
		\begin{bmatrix}
			\mathbf{0} & \cdots & \mathbf{0} &  \cdots  & \mathbf{0} & \cdots & \mathbf{M}_{1j}\bm{V}^{{\bm{\omega}}_{j}}_{1j} & \cdots  & \mathbf{0}
		\end{bmatrix}& i = 1\\
		\\
		\begin{bmatrix}
		\mathbf{0} & \cdots & \mathbf{M}_{ij}\bm{V}^{{\bm{\rho}}_{i}}_{ij} &  \cdots  & \mathbf{0} & \cdots & \mathbf{M}_{ij}\bm{V}^{{\bm{\omega}}_{j}}_{ij} & \cdots  & \mathbf{0} 
		\end{bmatrix}  & i \neq 1
	\end{cases}
\end{equation}
We call it a reduced  Jacobian matrix, because $\mathbf{M}_{ij}$ is a much smaller matrix than $\mathbf{C}_{ij}$. $\mathbf{C}_{ij}$ has   size $K_{ij} \times 13$. According to Lemma \autoref{lemma_C_ij}, $\mathbf{M}_{ij}$ has  size $4 \times 13$. Generally, $K_{ij}$ is much larger than 4.

We stack $\mathbf{J}_{ij}$ and $\mathbf{J}_{ij}^{r}$ to form the Jacobian matrix $\mathbf{J}$ and the reduced Jacobian matrix $\mathbf{J}^{r}$ for the cost function (\ref{equ_cost}), as
\begin{equation} \label{equ_J_Jr}
	\mathbf{J} = \begin{bmatrix}
		\vdots \\
		\mathbf{J}_{ij} \\
		\vdots
	\end{bmatrix}, \quad
	\mathbf{J}^{r} = \begin{bmatrix}
		\vdots \\
		\mathbf{J}_{ij}^{r} \\
		\vdots
	\end{bmatrix},
\end{equation}

The following lemma shows that  $	\mathbf{J}^{r} $ can replace 	$\mathbf{J} $ to calculate $\mathbf{J}^{T}\mathbf{J}$  in the  LM algorithm.
\begin{lemma} \label{lemma_JJ}
	For the planar BA, we have $\mathbf{J}^{T}\mathbf{J} = {\mathbf{J}^{r}}^{T}\mathbf{J}^{r}$.
\end{lemma}
\begin{proof}
	$\mathbf{J}$ and $\mathbf{J}^r$ are block vectors in terms of  $\mathbf{J}_{ij}$ and $\mathbf{J}_{ij}^r$ as defined in (\ref{equ_J_Jr}). According to block matrix multiplication, we have
	\begin{equation} \label{equ_sum_JTJ}
		\mathbf{J}^{T}\mathbf{J}  = \sum_{i,j}\mathbf{J}_{ij}^{T}\mathbf{J}_{ij}, \quad {\mathbf{J}^{r}}^{T}\mathbf{J}^{r}  = \sum_{i,j}{\mathbf{J}_{ij}^{r}}^{T}\mathbf{J}_{ij}^{r}
	\end{equation}
	For $i \neq 1$, using the expression in (\ref{equ_Jij_block}), $\mathbf{J}_{ij}^{T}\mathbf{J}_{ij}$ has the form
	\begin{equation} \label{equ_JTJ}
		\mathbf{J}_{ij}^{T}\mathbf{J}_{ij} = \begin{bmatrix}
			\mathbf{0} & \cdots  & 	\mathbf{0} & \cdots & \mathbf{0} & \cdots & \mathbf{0} \\
			\vdots        &  &  \vdots &  & \vdots &  & \vdots \\
			\mathbf{0} & \cdots  & 	 {\bm{V}^{{\bm{\rho}}_{i}}_{ij}}^{T}\mathbf{C}_{ij}^{T}\mathbf{C}_{ij}\bm{V}^{{\bm{\rho}}_{i}}_{ij} & \cdots &  {\bm{V}^{{\bm{\rho}}_{i}}_{ij}}^{T}\mathbf{C}_{ij}^{T}\mathbf{C}_{ij}\bm{V}^{{\bm{\omega}}_{j}}_{ij} & \cdots & \mathbf{0} \\
			\vdots        &   &  \vdots &  & \vdots &  & \vdots \\
			\mathbf{0} & \cdots  & 	 {\bm{V}^{{\bm{\omega}}_{j}}_{ij}}^{T}\mathbf{C}_{ij}^{T}\mathbf{C}_{ij}\bm{V}^{{\bm{\rho}}_{i}}_{ij} & \cdots &  {\bm{V}^{{\bm{\omega}}_{j}}_{ij}}^{T}\mathbf{C}_{ij}^{T}\mathbf{C}_{ij}\bm{V}^{{\bm{\omega}}_{j}}_{ij} & \cdots & \mathbf{0} \\
			\vdots        &   &  \vdots &  & \vdots &  & \vdots \\
			\mathbf{0} & \cdots  & 	\mathbf{0} & \cdots & \mathbf{0} & \cdots & \mathbf{0} \\
		\end{bmatrix}
	\end{equation}
	Similarly, using the expression in (\ref{equ_Jr_ij}), ${\mathbf{J}_{ij}^{r}}^{T}\mathbf{J}_{ij}^{r}$ has the form
	\begin{equation} \label{equ_JrTJr}
	{\mathbf{J}_{ij}^{r}}^{T}\mathbf{J}_{ij}^{r} = \begin{bmatrix}
	\mathbf{0} & \cdots  & 	\mathbf{0} & \cdots & \mathbf{0} & \cdots & \mathbf{0} \\
	\vdots        &  &  \vdots &  & \vdots &  & \vdots \\
	\mathbf{0} & \cdots  & 	 {\bm{V}^{{\bm{\rho}}_{i}}_{ij}}^{T}\mathbf{M}_{ij}^{T}\mathbf{M}_{ij}\bm{V}^{{\bm{\rho}}_{i}}_{ij} & \cdots &  {\bm{V}^{{\bm{\rho}}_{i}}_{ij}}^{T}\mathbf{M}_{ij}^{T}\mathbf{M}_{ij}\bm{V}^{{\bm{\omega}}_{j}}_{ij} & \cdots & \mathbf{0} \\
	\vdots        &   &  \vdots &  & \vdots &  & \vdots \\
	\mathbf{0} & \cdots  & 	 {\bm{V}^{{\bm{\omega}}_{j}}_{ij}}^{T}\mathbf{M}_{ij}^{T}\mathbf{M}_{ij}\bm{V}^{{\bm{\rho}}_{i}}_{ij} & \cdots &  {\bm{V}^{{\bm{\omega}}_{j}}_{ij}}^{T}\mathbf{M}_{ij}^{T}\mathbf{M}_{ij}\bm{V}^{{\bm{\omega}}_{j}}_{ij} & \cdots & \mathbf{0} \\
	\vdots        &   &  \vdots &  & \vdots &  & \vdots \\
	\mathbf{0} & \cdots  & 	\mathbf{0} & \cdots & \mathbf{0} & \cdots & \mathbf{0} \\
	\end{bmatrix}
	\end{equation}
		Substituting (\ref{equ_qr}) into $	\bm{V}_{{\bm{\rho}}_{i}}^{T}\mathbf{C}_{ij}^{T}\mathbf{C}_{ij}\bm{V}_{{\bm{\rho}}_{i}}$ and using the fact $\mathbf{Q}_{ij}^{T}\mathbf{Q}_{ij} = \mathbf{I}_{4}$, we have 
	\begin{equation} \label{equ_equvalent_1}
		\begin{split}
			{\bm{V}^{{\bm{\rho}}_{i}}_{ij}}^{T}\mathbf{C}_{ij}^{T}\mathbf{C}_{ij}\bm{V}^{{\bm{\rho}}_{i}}_{ij} & = {\bm{V}^{{\bm{\rho}}_{i}}_{ij}}^{T}\mathbf{M}_{ij}^{T}\left( \mathbf{Q}_{ij}^{T}\mathbf{Q}_{ij}\right) \mathbf{M}_{ij}\bm{V}^{{\bm{\rho}}_{i}}_{ij} \\
			 & = {\bm{V}^{{\bm{\rho}}_{i}}_{ij}}^{T}\mathbf{M}_{ij}^{T}\mathbf{M}_{ij}\bm{V}^{{\bm{\rho}}_{i}}_{ij} 
		\end{split}
	\end{equation}
	Similarly, we have 
	\begin{equation} \label{equ_equvalent_2}
		\begin{split}
			{\bm{V}^{{\bm{\rho}}_{i}}_{ij}}^{T}\mathbf{C}_{ij}^{T}\mathbf{C}_{ij}\bm{V}^{{\bm{\omega}}_{j}}_{ij} &= {\bm{V}^{{\bm{\rho}}_{i}}_{ij}}^{T}\mathbf{M}_{ij}^{T}\mathbf{M}_{ij}\bm{V}^{{\bm{\omega}}_{j}}_{ij}  \\
			{\bm{V}^{{\bm{\omega}}_{j}}_{ij}}^{T}\mathbf{C}_{ij}^{T}\mathbf{C}_{ij}\bm{V}^{{\bm{\rho}}_{i}}_{ij} &= {\bm{V}^{{\bm{\omega}}_{j}}_{ij}}^{T}\mathbf{M}_{ij}^{T}\mathbf{M}_{ij}\bm{V}^{{\bm{\rho}}_{i}}_{ij}  \\
			{\bm{V}^{{\bm{\omega}}_{j}}_{ij}}^{T}\mathbf{C}_{ij}^{T}\mathbf{C}_{ij}\bm{V}^{{\bm{\omega}}_{j}}_{ij} &= {\bm{V}^{{\bm{\omega}}_{j}}_{ij}}^{T}\mathbf{M}_{ij}^{T}\mathbf{M}_{ij}\bm{V}^{{\bm{\omega}}_{j}}_{ij} 
		\end{split}
	\end{equation}
	
	For $i = 1$, according to (\ref{equ_J1j}), the only  non-zero term for $\mathbf{J}_{1j}^{T}\mathbf{J}_{1j}$   is ${\bm{V}^{{\bm{\omega}}_{j}}_{1j}}^{T}\mathbf{C}_{1j}^{T}\mathbf{C}_{1j}\bm{V}^{{\bm{\omega}}_{j}}_{1j}$. On the other hand, according to (\ref{equ_Jr_ij}), ${\mathbf{J}_{1j}^{r}}^{T}\mathbf{J}_{1j}^{r}$ has only one corresponding non-zero term ${\bm{V}^{{\bm{\omega}}_{j}}_{1j}}^{T}\mathbf{M}_{1j}^{T}\mathbf{M}_{1j}\bm{V}^{{\bm{\omega}}_{j}}_{1j} $. Similar to the derivation  in (\ref{equ_equvalent_1}), we have 
	\begin{equation} \label{equ_equvalent_3}
		{\bm{V}^{{\bm{\omega}}_{j}}_{1j}}^{T}\mathbf{C}_{1j}^{T}\mathbf{C}_{1j}\bm{V}^{{\bm{\omega}}_{j}}_{1j} = {\bm{V}^{{\bm{\omega}}_{j}}_{1j}}^{T}\mathbf{M}_{1j}^{T}\mathbf{M}_{1j}\bm{V}^{{\bm{\omega}}_{j}}_{1j}.
	\end{equation}
	
	In summary, 
	using (\ref{equ_equvalent_1}),  (\ref{equ_equvalent_2}) and (\ref{equ_equvalent_3}) , we have $\mathbf{J}_{ij}^{T}\mathbf{J}_{ij} = {\mathbf{J}_{ij}^{r}}^{T}\mathbf{J}_{ij}^{r}$. According to (\ref{equ_sum_JTJ}), consequently we have $\mathbf{J}^{T}\mathbf{J} = {\mathbf{J}^{r}}^{T}\mathbf{J}^{r}$.
\end{proof}

\subsection{Reduced Residual  Vector }
Let us define the  residual vector for the $K_{ij}$ points in $\mathbb{P}_{ij}$ as $\bm{\delta}_{ij} = \left[\delta_{ij1}, \, \delta_{ij2}, \quad \cdots \quad \delta_{ijK_{ij}} \right]^{T}$. According to (\ref{equ_delta_exp}) and (\ref{equ_C_ij}), $\bm{\delta}_{ij}$ can be written as
\begin{equation} \label{equ_delta_ij}
\bm{\delta}_{ij} =	\mathbf{C}_{ij}\bm{\nu}_{ij}.
\end{equation}
We define the \textbf{ reduced residual vector} $\bm{\delta}_{ij}^{r}$  of $\bm{\delta}_{ij}$ as
\begin{equation} \label{equ_red_delta_ij}
\bm{\delta}_{ij}^{r} =	\mathbf{M}_{ij}\bm{\nu}_{ij}
\end{equation}
 Stacking   all $\bm{\delta}_{ij}$ and $\bm{\delta}_{ij}^{r}$, we get the  residual vector $\bm{\delta}$  and the reduced residual vector $\bm{\delta}^{r}$  as 
\begin{equation} \label{equ_delta_deltar}
	\bm{\delta} = \begin{bmatrix}
		\vdots \\
		\bm{\delta}_{ij} \\
		\vdots
	\end{bmatrix}, \quad
	\bm{\delta}^{r} = \begin{bmatrix}
		\vdots \\
		\bm{\delta}_{ij}^{r} \\
		\vdots
	\end{bmatrix}.
\end{equation}
The following lemma shows that  $	\bm{\delta}^{r} $ can replace 	$\bm{\delta}$ in the  LM algorithm.
\begin{lemma} \label{lemma_J_delta}
	 For the planar BA, we have $\mathbf{J}^{T}\bm{\delta} = {\mathbf{J}^r}^{T}\bm{\delta}^{r} $. 
\end{lemma}
\begin{proof}
	$\mathbf{J}$, $\mathbf{J}^{r}$,  $\mathbf{\bm{\delta}}$ and $\mathbf{\bm{\delta}^{r}}$ are block vectors with elements $\mathbf{J}_{ij}$,  $\mathbf{J}_{ij}^{r}$, $\mathbf{\bm{\delta}}_{ij}$ and $\mathbf{\bm{\delta}}_{ij}^{r}$  as defined in (\ref{equ_J_Jr}) and (\ref{equ_delta_deltar}), respectively. Applying the block matrix multiplication, we have 
	\begin{equation} \label{equ_J_delta}
		\mathbf{J}^{T}\bm{\delta} = \sum_{i,j}\mathbf{J}_{ij}^{T}\bm{\delta}_{ij}, \quad {\mathbf{J}^{r}}^{T}\bm{\delta}^{r} = \sum_{i,j}{\mathbf{J}_{ij}^{r}}^{T}\bm{\delta}_{ij}^{r}
	\end{equation}
	
	For $i \neq 1$, using the expression of $\mathbf{J}_{ij}$ in (\ref{equ_Jij_block}) and $\mathbf{J}_{ij}^{r}$ in (\ref{equ_Jr_ij}), and the expression of $\bm{\delta}_{ij}$ in (\ref{equ_delta_ij}) and $\bm{\delta}_{ij}^{r}$ in (\ref{equ_red_delta_ij}),  $\mathbf{J}_{ij}^{T}\bm{\delta}_{ij}$ and ${\mathbf{J}_{ij}^{r}}^{T}\bm{\delta}_{ij}^{r}$ have the forms as
	\begin{equation} \label{equ_J_delta_Jr_delta}
		\mathbf{J}_{ij}^{T}\bm{\delta}_{ij} =\begin{bmatrix}
				\mathbf{0} \\
				\vdots        \\
				{\bm{V}^{\bm{\rho}_{i}}_{ij}}^{T}\mathbf{C}_{ij}^{T}\mathbf{C}_{ij}\bm{\nu}_{ij}\\
				\vdots        \\
				{\bm{V}^{\bm{\omega}_{j}}_{ij}}^{T}\mathbf{C}_{ij}^{T}\mathbf{C}_{ij}\bm{\nu}_{ij} \\
				\vdots \\
				\mathbf{0}
		\end{bmatrix}, \quad 
		{\mathbf{J}_{ij}^{r}}^{T}\bm{\delta}_{ij}^{r} =\begin{bmatrix}
			\mathbf{0} \\
			\vdots        \\
			{\bm{V}^{\bm{\rho}_{i}}_{ij}}^{T}\mathbf{M}_{ij}^{T}\mathbf{M}_{ij}\bm{\nu}_{ij} \\
			\vdots        \\
			{\bm{V}^{\bm{\omega}_{j}}_{ij}}^{T}\mathbf{M}_{ij}^{T}\mathbf{M}_{ij}\bm{\nu}_{ij} \\
			\vdots \\
			\mathbf{0}
		\end{bmatrix} 
	\end{equation}
	Substituting (\ref{equ_qr}) into 	${\bm{V}^{\bm{\rho}_{i}}_{ij}}^{T}\mathbf{C}_{ij}^{T}\mathbf{C}_{ij}\bm{\nu}_{ij}$ and  using the fact $\mathbf{Q}_{ij}^{T}\mathbf{Q}_{ij} = \mathbf{I}_{4}$, we have 
	\begin{equation} \label{equ_equvalent_delta_1}
		\begin{split}
			{\bm{V}^{\bm{\rho}_{i}}_{ij}}^{T}\mathbf{C}_{ij}^{T}\mathbf{C}_{ij}\bm{\nu}_{ij} & = {\bm{V}^{{\bm{\rho}}_{i}}_{ij}}^{T}\mathbf{M}_{ij}^{T}\left( \mathbf{Q}_{ij}^{T}\mathbf{Q}_{ij}\right) \mathbf{M}_{ij}\bm{\nu}_{ij} \\
			& = {\bm{V}^{{\bm{\rho}}_{i}}_{ij}}^{T}\mathbf{M}_{ij}^{T}\mathbf{M}_{ij}\bm{\nu}_{ij}
		\end{split}
	\end{equation}
	Similarly, we have 
	\begin{equation} \label{equ_equvalent_delta_2}
		{\bm{V}^{\bm{\omega}_{j}}_{ij}}^{T}\mathbf{C}_{ij}^{T}\mathbf{C}_{ij}\bm{\nu}_{ij} = {\bm{V}^{\bm{\omega}_{j}}_{ij}}^{T}\mathbf{M}_{ij}^{T}\mathbf{M}_{ij}\bm{\nu}_{ij}
	\end{equation}
	
	For $ i = 1$, substituting (\ref{equ_J1j}) and (\ref{equ_delta_ij}) into $\mathbf{J}_{1j}^{T}\bm{\delta}_{1j}$ and applying the block matrix multiplication, we find the only non-zero term of $\mathbf{J}_{1j}^{T}\bm{\delta}_{1j}$ is ${\bm{V}^{\bm{\omega}_{j}}_{ij}}^{T}\mathbf{C}_{1j}^{T}\mathbf{C}_{1j}\bm{\nu}_{1j} $. On the other hand, substituting (\ref{equ_Jr_ij}) and (\ref{equ_red_delta_ij}) into ${\mathbf{J}_{1j}^{r}}^{T}\bm{\delta}_{1j}^{r}$, we find that ${\mathbf{J}_{1j}^{r}}^{T}\bm{\delta}_{1j}^{r}$ only has one non-zero term ${\bm{V}^{\bm{\omega}_{j}}_{ij}}^{T}\mathbf{M}_{1j}^{T}\mathbf{M}_{1j}\bm{\nu}_{1j} $. Similar to the derivation  in (\ref{equ_equvalent_delta_1}), we have
	\begin{equation}\label{equ_equvalent_delta_3}
		{\bm{V}^{\bm{\omega}_{j}}_{ij}}^{T}\mathbf{C}_{1j}^{T}\mathbf{C}_{1j}\bm{\nu}_{1j}  = {\bm{V}^{\bm{\omega}_{j}}_{ij}}^{T}\mathbf{M}_{1j}^{T}\mathbf{M}_{1j}\bm{\nu}_{1j} 
	\end{equation}
	
	In summary, from  (\ref{equ_equvalent_delta_1}), (\ref{equ_equvalent_delta_2}) and (\ref{equ_equvalent_delta_3}), we have $\mathbf{J}_{1j}^{T}\bm{\delta}_{1j} = {\mathbf{J}_{1j}^{r}}^{T}\bm{\delta}_{1j}^{r}$. According to (\ref{equ_J_delta}), we finally get $\mathbf{J}^{T}\bm{\delta} = {\mathbf{J}^r}^{T}\bm{\delta}^{r}$. 
\end{proof}



\subsection{Planar Bundle Adjustment Algorithm}
According to Lemma \autoref{lemma_C_ij}, Lemma \autoref{lemma_JJ} and Lemma \ref{lemma_J_delta}, we have the following theorem.

\begin{theorem} \label{theorem_LM_reduce}
	For the planar BA, $\mathbf{J}^{r}$ and $\bm{\delta}^{r}$ can replace $\mathbf{J}$ and $\bm{\delta}$ in   (\ref{equ_lm_update}) to compute the step in the LM algorithm, and  each block $\mathbf{J}^{r}_{ij}$ and $\bm{\delta}^{r}_{ij}$ in $\mathbf{J}^{r}$ and $\bm{\delta}^{r}$ has 4 rows.
\end{theorem}
\begin{proof}
	LM algorithm uses (\ref{equ_lm_update}) to calculate the step for each iteration. According to Lemma \autoref{lemma_JJ}, we have $\mathbf{J}^{T}\mathbf{J} = {\mathbf{J}^{r}}^{T}\mathbf{J}^{r}$. Besides, based on  Lemma \ref{lemma_J_delta}, we have $\mathbf{J}^{T}\bm{\delta} = {\mathbf{J}^r}^{T}\bm{\delta}^{r} $. Consequently, we have  $\left( {\mathbf{J}^r}^{T} \mathbf{J}^r+\lambda\mathbf{I}\right) \bm{\xi} =  {\mathbf{J}^{r}}^{T} \bm{\delta}^{r} $  is equivalent to  $\left( \mathbf{J}^{T} \mathbf{J}+\lambda\mathbf{I}\right) \bm{\xi} =  \mathbf{J}^{T} \bm{\delta}$. Thus $\mathbf{J}^{r}$ and $\bm{\delta}^{r}$ can replace $\mathbf{J}$ and $\bm{\delta}$ for computing the step in the LM algorithm.
	
	According to the definition of  $\mathbf{J}^{r}_{ij}$ in (\ref{equ_Jr_ij}) and $\bm{\delta}^{r}_{ij}$ in (\ref{equ_red_delta_ij}), we know that the number of rows of  $\mathbf{J}^{r}_{ij}$ and  $\bm{\delta}^{r}_{ij}$ is the same as the number of rows of $\mathbf{M}_{ij}$. According to Lemma \autoref{lemma_C_ij},  $\mathbf{M}_{ij}$ has 4 rows. Consequently, we have that  $\mathbf{J}^{r}_{ij}$ and $\bm{\delta}^{r}_{ij}$ have 4 rows. 
\end{proof}

As mentioned in  \autoref{theorem_LM_reduce}, no matter how many points are  in $\mathbb{P}_{ij}$, the reduced  $\mathbf{J}^{r}_{ij}$ and $\bm{\delta}_{ij}^{r}$ always have 4 rows. This significantly reduces the computational cost in the LM algorithm. Specifically, we have the following corollary:

\begin{corollary} \label{corollary_reduce_time}
The runtime for computing $\mathbf{J}_{ij}^r$, $\bm{\delta}_{ij}^r$, ${\mathbf{J}_{ij}^{r}}^{T}\mathbf{J}_{ij}^{r}$ and ${\mathbf{J}_{ij}^{r}}^{T}\bm{\delta}_{ij}^r$ is  $\frac{4}{K_{ij}}$ relative to computing  the original $\mathbf{J}_{ij}$, $\bm{\delta}_{ij}$, $\mathbf{J}_{ij}^{T}\mathbf{J}_{ij}$ and $\mathbf{J}_{ij}^{T}\bm{\delta}_{ij}$, respectively. 
\end{corollary}
\begin{proof}
	From the expressions for $\mathbf{J}^{r}_{ij}$ in (\ref{equ_Jr_ij}) and $\mathbf{J}_{ij}$ in (\ref{equ_Jij}) and (\ref{equ_J1j}), we know that the difference between $\mathbf{J}^{r}_{ij}$ and  $\mathbf{J}_{ij}$ is that we use $\mathbf{M}_{ij}$ to replace $\mathbf{C}_{ij}$. $\mathbf{M}_{ij}$ has 4 rows and  $\mathbf{C}_{ij}$ has $K_{ij}$ rows. Thus, the runtime for computing $\mathbf{J}^{r}_{ij}$ is $\frac{4}{K_{ij}}$  relative to computing $\mathbf{J}_{ij}$. Similarly, the runtime for computing $\bm{\delta}^{r}_{ij}$ is $\frac{4}{K_{ij}}$  relative to computing $\bm{\delta}_{ij}$.

According to \autoref{theorem_LM_reduce}, $\mathbf{J}_{ij}^{r}$ and $\bm{\delta}_{ij}^{r}$ each has 4 rows, and $\mathbf{J}_{ij}$ and $\bm{\delta}_{ij}$ each  has $K_{ij}$ rows. According to the rules of matrix multiplication,   the runtime for computing ${\mathbf{J}_{ij}^{r}}^{T}\mathbf{J}_{ij}^{r}$ and ${\mathbf{J}_{ij}^{r}}^{T}\bm{\delta}_{ij}^r$ is $\frac{4}{K_{ij}}$ relative to computing  $\mathbf{J}_{ij}^{T}\mathbf{J}_{ij}$ and $\mathbf{J}_{ij}^{T}\bm{\delta}_{ij}$, respectively.
\end{proof}

The additional  cost here is to calculate $\mathbf{C}_{ij}=\mathbf{Q}_{ij}\mathbf{M}_{ij}$. As $\mathbf{C}_{ij}$ keeps constant during the iteration, we only need to compute it once before the iteration. As shown in our experimental results, this step only slightly increases the computational time for  the pose graph initialization step. 
We summarize our PBA in Algorithm \autoref{algo_pba}.
\begin{algorithm} \label{algo_pba}
	\SetAlgoLined
	\KwInput {Initial guess of $N-1$ poses  and $M$ plane parameters, and the  measurements $\left\lbrace \mathbb{P}_{ij} \right\rbrace$.}
	\KwResult{Refined poses and plane parameters.}
	\tcp{Initialization}
	Calculate $\mathbf{c}_{ijk}$ for each $\mathbf{p}_{ijk} \in \mathbb{P}_{ij}$ as (\ref{equ_delta_definition})\;
	Stack $\mathbf{c}_{ijk}$ to get $\mathbf{C}_{ij}$ as (\ref{equ_C_ij})\;
	Compute the factorization $\mathbf{C}_{ij}=\mathbf{Q}_{ij}\mathbf{M}_{ij}$ as mentioned in Lemma \autoref{lemma_C_ij}\;
	\tcp{Iterative Refine}
	\While{not converge}
		{ Compute the reduced Jacobian matrix block $\mathbf{J}_{ij}^{r}$ in (\ref{equ_Jr_ij}) and the reduced residual block $\bm{\delta}_{ij}^{r}$ in (\ref{equ_red_delta_ij})\;
		   Stack them to form $\mathbf{J}^{r}$ and $\bm{\delta}^{r}$\;
		   Use the LM algorithm to update  current estimate\;
		}  
			\caption{Planar Bundle Adjustment}
\end{algorithm}

\textbf{Combine Planes with Other Features} \quad  
Planes are sometimes together with other features, such as points \cite{taguchi2013point,wang2019submap,grant2019efficient,yang2019tightly,li2019robust}. In the  cost function derived from multiple features, the Jacobian matrix  from the plane cost would have the same form as (\ref{equ_Jij_block}) and (\ref{equ_J1j}), and the  residual vector  would also have the same form as (\ref{equ_delta_ij}). Therefore,  they can  be replaced by our reduced Jacobian matrix and reduced residual vector in the BA with    multiple features.

\begin{figure*}[tb]
	\centering 
	\includegraphics[width=\textwidth]{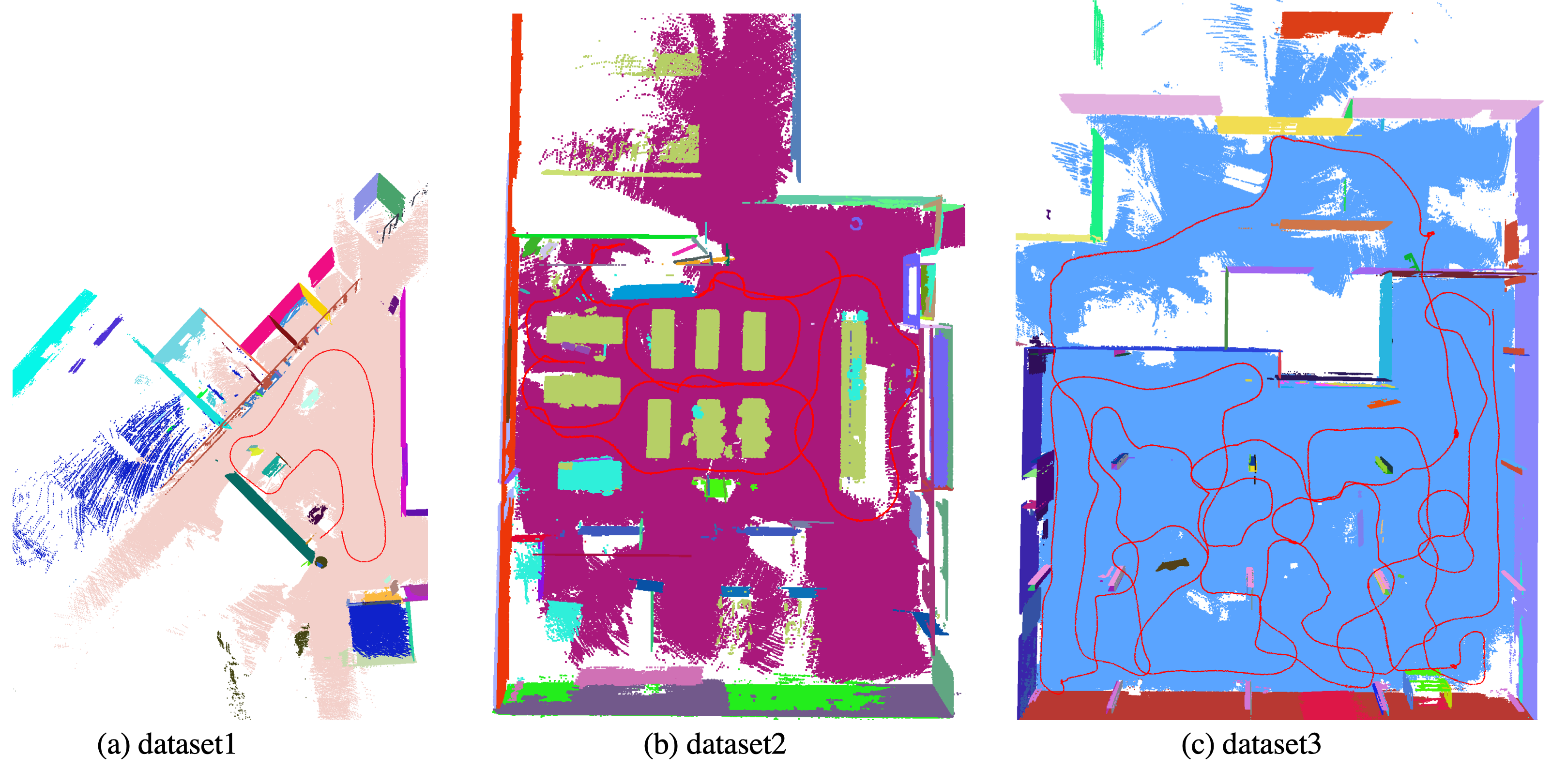}  
	\caption{The three datasets used in our experiments.} 
	\label{fig:dataset}
\end{figure*}

\begin{figure*}[!h]
	\centering 
	\includegraphics[width=\textwidth]{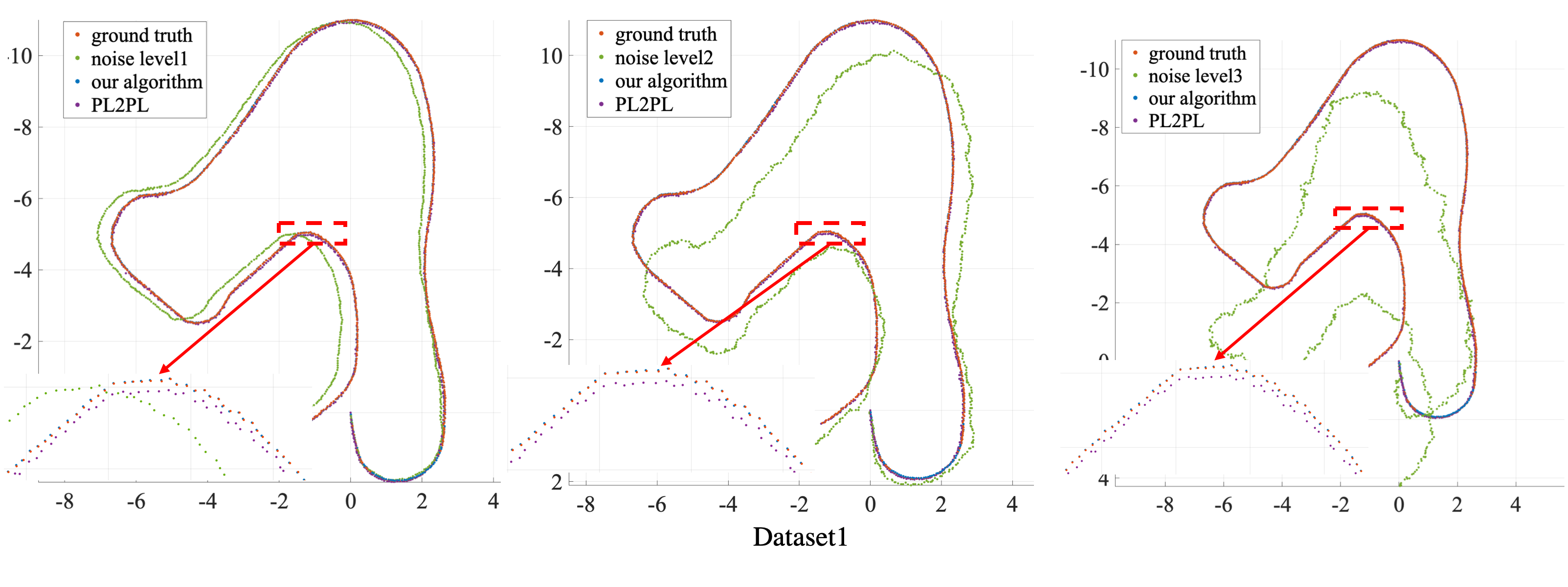}
	\includegraphics[width=\textwidth]{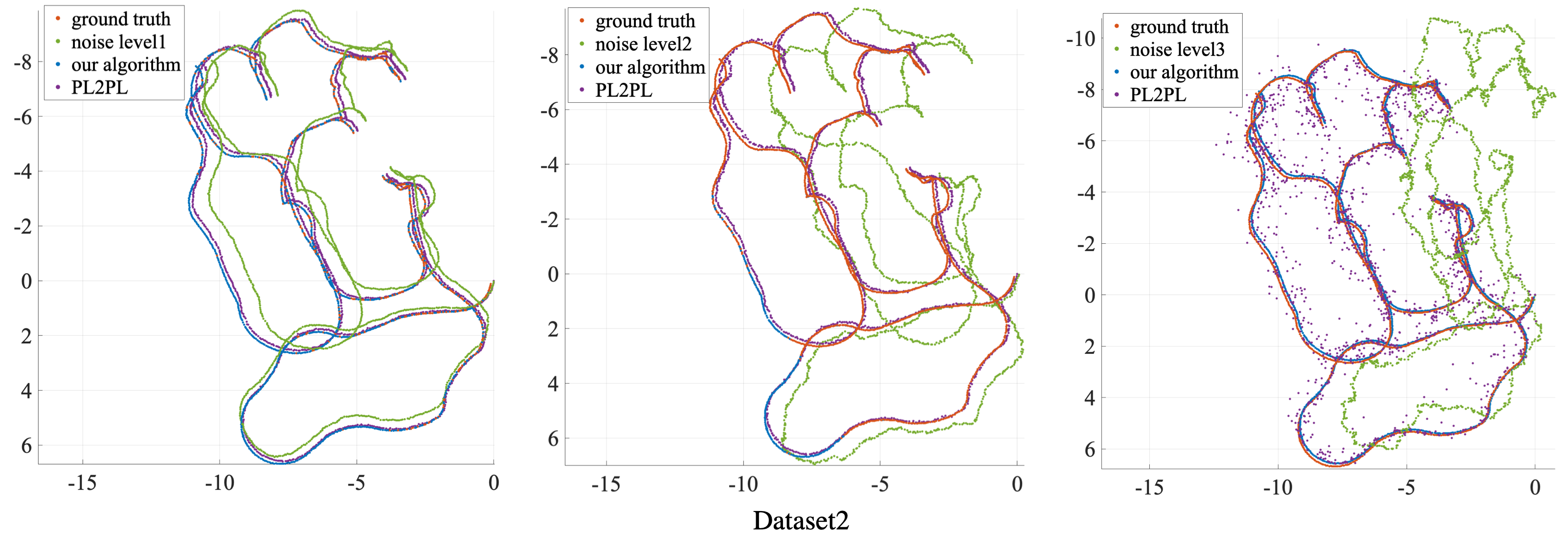}
	\includegraphics[width=\textwidth]{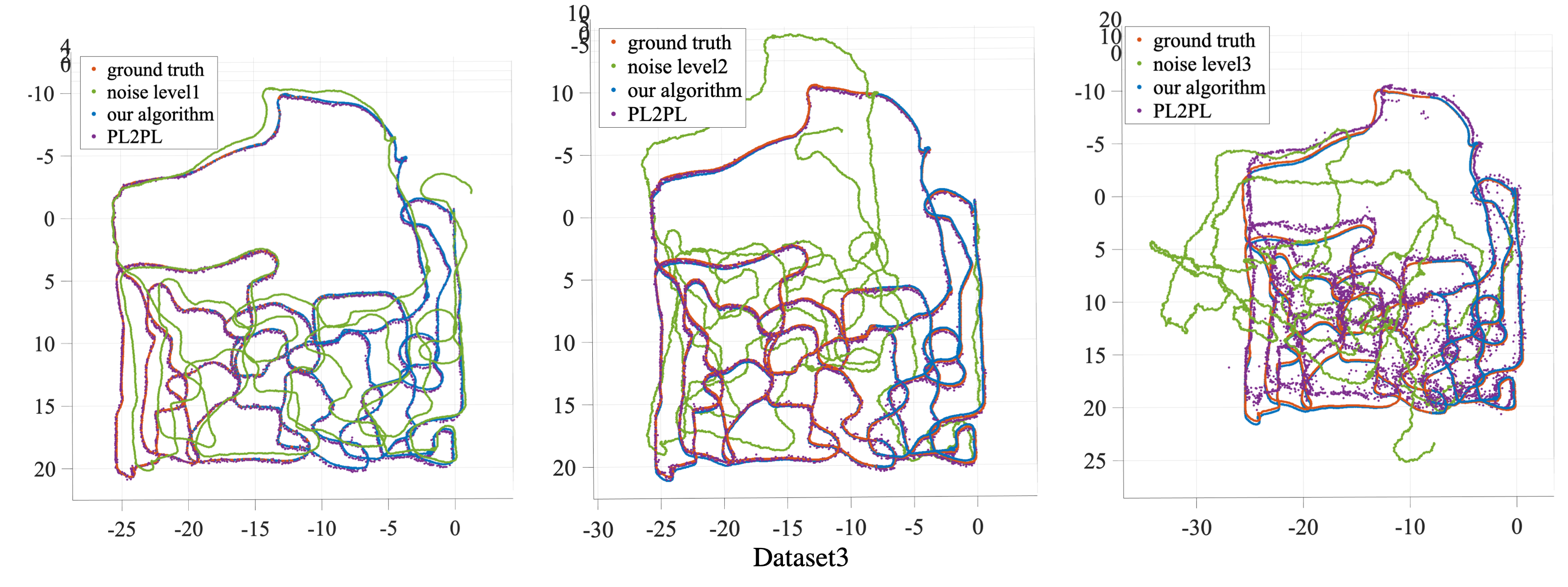} 
	\caption{Experimental results of our algorithm and PL2PL \cite{geneva2018lips} with varying initialization noise levels. PL2PL is not stable to large initial errors.} 
	\label{fig:results}
\end{figure*}

\begin{table}[tb]
	\caption{The characteristics of the 3 datasets.}
	\label{tab:datasets}
	\scriptsize%
	\centering%
	\begin{tabular}{c c c c c c}
		\toprule
		dataset & \#poses &   \#planes &   \#points &   length ($m$)  \\
		\hline
		dataset1 & 695   & 154  & $6.98 \times 10^6$ &  43.2 \\
		dataset2 & 1781 & 370 & $16.82 \times 10^6$ &  104.4   \\
		dataset3 & 6547 &  591& $68.99 \times 10^6$  & 403.5\\
		\bottomrule
	\end{tabular}%
\end{table}

\begin{table*}[tb]
	\caption{Experimental results for different algorithms.  The column  QR shows the  time for the QR decomposition of our algorithm. The column Init. presents the time for initializing the Ceres solver. The column Optimization describes the total time for the LM algorithm. The column Per Iter. lists the average runtime for each LM iteration. Our algorithm has a lower initialization time.  The computational cost for QR decomposition is ignorable  compared to the significant gain from the optimization process. Specifically, our algorithm is about 74, 49 and  107 times faster than DPT2PL on the three datasets, respectively. Here we only evaluate DPT2PL on Noise Level 1 for the 3 datasets, as we only consider  the computational time of DPT2PL. For a specific dataset,  the  runtime of one iteration for different  noise levels  is similar. Our algorithm is more accurate, has lower computational complexity and converges faster than PL2PL\cite{geneva2018lips}.} 
	\label{tab:results}
	\scriptsize%
	\centering%
	\begin{tabular}{ c | c l c  c c | c c | c c c c}
		\toprule
		\multirow{2}{*}{Dataset} & \multirow{2}{*}{Noise Level}  & \multirow{2}{*}{Method}  &   \multirow{2}{*}{\#Iter}   &  \multicolumn{2}{c}{ATE}  &  \multicolumn{2}{c}{ Cost in Eq. (\ref{equ_cost}) }  &  \multicolumn{4}{c}{Time (s)} \\ \cline{5-6} \cline{7-8} \cline{9-12}
		& & & &$\text{ATE}_{\mathbf{R}}$ (\si{\degree}) & $\text{ATE}_{\mathbf{t}}$ ($m$)& Initial  &  Final & QR &  Init.& Optimization & Per Iter.\\
		\hline
		\multirow{7}{*}{dataset1}  
		& \multirow{3}{*}{level 1} 
		& Ours & 81& $4.44 \times 10^{-2}$ &  $3.07 \times 10^{-4}$  &  $2.38 \times 10^5$ & $1.42 \times 10^3$ & 1.16 & 0.016 & 11.06 & 0.14\\
		&  &DPT2PL & 81 &  $4.44 \times 10^{-2}$ &  $3.07 \times 10^{-4}$  &  $2.38 \times 10^5$ & $1.42 \times 10^3$ & 0 & 6.26 & 823.31 &  10.16\\ 
		&  &PL2PL\cite{geneva2018lips} &  384 & 0.58  & 0.28 & $2.38 \times 10^5$ &$1.98 \times 10^4$ & 0 & 0.026 & 69.53 & 0.18\\  \cline{2-12}
		& \multirow{2}{*}{level 2} 
		& Ours & 176 & $4.68 \times 10^{-2}$ & $ 3.98 \times 10^{-4}$  &  $1.78 \times 10^6$ & $1.42 \times 10^3$ & 1.18 & 0.017 & 22.86 & 0.13  \\
		&  &PL2PL\cite{geneva2018lips} &  1000 & 0.58 &  0.28 & $1.78 \times 10^6$ & $1.98 \times 10^4$ & 0 & 0.025 & 185.97 & 0.19\\ \cline{2-12}
		& \multirow{2}{*}{level 3} 
		& Ours &   224 & $4.96 \times 10^{-2}$ & $ 4.22 \times 10^{-4}$ & $1.78 \times 10^6$ & $1.42 \times 10^3$ & 1.15 & 0.016 & 28.56 & 0.13\\
		&  &PL2PL\cite{geneva2018lips} & 1000   & 0.58 & 0.28  & $1.78 \times 10^6$ & $1.98 \times 10^4$ & 0 & 0.024 & 165.64 & 0.17\\
		\hline
		\multirow{7}{*}{dataset2}
		& \multirow{3}{*}{level 1} 
		& Ours & 75 & $4.88 \times 10^{-2}$ & $8.57 \times 10^{-5}$ & $7.27 \times 10^5$ & $3.88 \times 10^3$ & 2.88 & 0.060 & 37.90 & 0.51\\
		& & DPT2PL & 75 & $4.88 \times 10^{-2}$ & $8.57 \times 10^{-5}$ & $7.27 \times 10^5$ & $3.88 \times 10^3$ & 0 & 16.46 & 1843.87 & 24.58\\
		&  &PL2PL\cite{geneva2018lips} &  203 & 0.95 & 0.16 & $7.27 \times 10^5$ & $8.46 \times 10^3$ & 0 & 0.11 & 219.45 & 1.08\\
		\cline{2-12}
		& \multirow{2}{*}{level 2} 
		& Ours &  119 & $5.76 \times 10^{-2}$ & $2.32 \times 10^{-4}$  & $4.38 \times  10^7$ &  $3.90 \times 10^3$ & 2.89 & 0.068 & 63.62 & 0.53  \\
		&  &PL2PL\cite{geneva2018lips} &  384 & 1.09  &  0.20&  $4.38 \times 10^7$ & $1.35 \times 10^4$ & 0 & 0.11 & 413.91 & 1.07\\
		\cline{2-12}
		& \multirow{2}{*}{level 3} 
		& Ours &  723 &  $7.61 \times 10^{-2}$ & $1.02 \times 10^{-2}$ & $8.38 \times 10^7$ &  $4.23 \times 10^3$ & 2.86 & 0.062 & 365.70 & 0.51\\
		&  &PL2PL\cite{geneva2018lips} &   186 & 2.56  & 0.97 & $8.38 \times 10^7$ & $1.32 \times 10^6$& 0 & 0.12 &  195.56 & 1.05\\
		\hline
		\multirow{7}{*}{dataset3}  
		& \multirow{3}{*}{level 1} 
		& Ours &   503 & $4.28 \times 10^{-2}$& $3.17 \times 10^{-4}$ & $2.81 \times 10^7$ & $1.28 \times10^4$ &  11.14  & 0.17 & $4.87 \times 10^2$  &  0.97 \\
		&  &DPT2PL &  503 & $4.28 \times 10^{-2}$& $3.17 \times 10^{-4}$ & $2.81 \times 10^7$ & $1.28 \times10^4$ &  0 & 77.03 & $5.22 \times 10^4$  & 103.69  \\
		&  &PL2PL\cite{geneva2018lips} &  1000 & 0.90& 0.41 & $2.81 \times 10^7$  & $4.00 \times 10^5$ & 0 & 0.35& $2.45\times 10^3$ & 2.45\\ \cline{2-12}
		& \multirow{2}{*}{level 2} 
		& Ours &  883 & $6.56 \times 10^{-2}$ & $2.32 \times 10^{-2}$ &  $5.99 \times 10^8$ & $1.39 \times 10^4$ & 10.89 & 0.16 & $8.38 \times 10^2$ & 0.95\\
		& &PL2PL\cite{geneva2018lips} & 1000 & 0.95  &  0.47 &  $5.99 \times 10^8$ & $7.02 \times 10^5$ & 0 &  0.36& $2.49 \times 10^3$  & 2.49\\ \cline{2-12}
		& \multirow{2}{*}{level 3} 
		& Ours &  13,385 & $8.85 \times 10^{-2}$ & $6.22 \times 10^{-2}$& $3.97 \times 10^9$& $1.40 \times 10^4$ & 11.32 & 0.17 &  $1.42 \times 10^4$ & 1.06 \\
		& & PL2PL\cite{geneva2018lips}  & 15,000 & 8.66  & 1.34  & $3.97 \times 10^9$ & $1.02 \times 10^7$& 0 & 0.33 & $3.77 \times 10^4$ & 2.51
		\\
		\bottomrule
	\end{tabular}%
\end{table*}

\section{Experiments}

In this section, we  evaluate the performance of our algorithm.   We compare our algorithm against the direct solution for the point-to-plane cost (\ref{equ_cost}) using the traditional BA framework (\textbf{DPT2PL}), and the state-of-the-art solution \cite{geneva2018lips}  using   plane-to-plane cost (\textbf{PL2PL}). We evaluate  the accuracy, computational time, convergence speed and the robustness to  errors of the initialization in the compared algorithms. 

\textbf{Datasets} \quad We collected 3 indoor datasets using the NavVis M6 device \footnote{https://www.navvis.com/m6}. The NavVis M6 estimates the device pose using  1 multi-layer  Velodyne LiDAR, 3 single-layer LiDAR,  IMU, WiFi signals as well as ground control points. It provides an accurate trajectory and a dense point cloud with  an accuracy on the order of centimeters. We use the recordings of the Velodyne LiDAR \footnote{https://velodynelidar.com/} and the trajectory from the NavVis M6 to build the datasets. We sampled the trajectory so that  the distance between two poses is larger than $5cm$.
Planes are detected for each recording of the Velodyne LiDAR by the region-growing method introduced in \cite{poppinga2008fast}. Then we use the known pose to get the plane-plane data association. Specifically,  the global planes are initialized by the planes detected in the first pose. We track and grow the global planes frame by frame. Local planes of  the latest frame are first transformed into the global coordinate system using the known pose.  We calculate the  distances between   points of a local plane  and each of the global planes, and then we  select the global plane that has the shortest mean point-to-plane distance.  A match occurs when  a local plane and the closest global plane have a mean point-to-plane distance   smaller than $5cm$, and the angle between the normals of the two planes is less than 10\si{\degree}. The points  of  the local plane whose distances are less than $5cm$ are then  added into the global plane, and the parameters of the global plane are then updated. Unmatched local planes that have more than 50 points are recognized as new planes, and are  added into the global plane list for  tracking during future frames.  \autoref{fig:dataset} shows the 3 datasets. The characteristics of the 3 datasets are listed in \autoref{tab:datasets}.

\textbf{Initialization Error} \quad The LM algorithm requires an initial set of values for the unknowns.  We consider  how the error in  the initial estimation affects the performance of the different algorithms.  
Specifically, we perturb the rotation matrix by an error  rotation matrix represented by the Euler angles which are sampled from a zero-mean Gaussian distribution with standard deviation (STD) $\delta_{\mathbf{R}}$. Additionally, we add zero-mean Gaussian noise with STD $\delta_{\mathbf{t}}$ to the translation.  We add  noise to the trajectory as follows. To simplify the notation, we use $\mathbf{T} = \begin{bmatrix}
\mathbf{R} & \mathbf{t} \\
\mathbf{0} & 1
\end{bmatrix}$  to represent a rigid transformation. We first consider adding noise to two consecutive poses $\mathbf{T}_i$ and $\mathbf{T}_{i+1}$.
 Assume $\mathbf{T}_i^{i+1}$ is the relative pose between $\mathbf{T}_i$ and $\mathbf{T}_{i+1}$, and $\hat{\mathbf{T}}_i $ and $\hat{\mathbf{T}}_{i+1}$ are the perturbed poses, respectively. We perturb $\mathbf{T}_i$ and $\mathbf{T}_{i+1}$ by $\mathbf{T}_i^{err}$ and $\mathbf{T}_{i+1}^{err}$, respectively,  so that $\hat{\mathbf{T}}_i = \mathbf{T}_i^{err}\mathbf{T}_i$ and $\hat{\mathbf{T}}_{i+1} = \mathbf{T}_{i+1}^{err}\mathbf{T}_i^{i+1}\hat{\mathbf{T}}_i$. It  is clear that the noise $\mathbf{T}_i^{err}$ on $\mathbf{T}_{i}$ will also affect $\mathbf{T}_{i+1}$. We start this process from $i=1$. Thus, the noise  will accumulate along the trajectory. We consider three noise levels  as listed below:
\begin{itemize}
	\item Noise Level 1: $\delta_{\mathbf{R}} = 0.1$\si{\degree}, $\delta_{\mathbf{t}} = 0.01m$
	\item Noise Level 2: $\delta_{\mathbf{R}} = 0.5$\si{\degree}, $\delta_{\mathbf{t}} = 0.03m$
	\item Noise Level 3: $\delta_{\mathbf{R}} = 1.0$\si{\degree}, $\delta_{\mathbf{t}} = 0.05m$
\end{itemize}
Although the noise is small in terms of one pose, the noise  accumulating along the trajectory  may yield a large error. Generally, a longer trajectory will lead to a larger error. \autoref{fig:results} shows the perturbed as well as  the original trajectories for the 3 datasets. We use the perturbed trajectory to initialize the poses.  For  planes,  we first calculate their local parameters from the local point clouds at the first poses that observe them. We initialize the global plane parameters by transforming the local parameters into the global coordinate system using the perturbed poses.  

\textbf{Experiment Setup} \quad In the experiments, all  algorithms use the same parameterizations for the  rotation matrix and the plane, \textit{i.e.}, angle-axis and CP  parameterization \cite{geneva2018lips}, respectively. We ran the experiments on a computer with an Intel(R) Xeon(R)  E5-2620 2.10GHz CPU and 80G memory. We adopted Ceres \cite{ceres-solver} as the nonlinear least-squares solver, and  used the  Schur complement trick \cite{triggs1999bundle} to solve the linear system (\ref{equ_lm_block}) (specifically, SPARSE\_SCHUR linear solver in Ceres).  We set the function  and parameter tolerance to $10^{-10}$, and set the maximum number of iterations to 1000, except for the dataset3 at noise level 3. In this case, the initial error is very large, and more iterations are required. We set the maximum number of iterations to 15,000 for dataset3 at noise level 3.  

We use the perturbed trajectory to initialize the LM algorithm, and compare the original trajectory with the results from the LM algorithm. We employ the absolute trajectory error (ATE) \cite{zhang2018tutorial} to quantify the accuracy of the result. Specifically, for the $k$th pose, given the ground truth $\mathbf{R}_k$  and $\mathbf{t}_k$ and  the estimation $\hat{\mathbf{R}}_k$  and $\hat{\mathbf{t}}_k$, we calculate  $\Delta\mathbf{R}_k$ and $\Delta\mathbf{t}_k$ as
\begin{equation}
\begin{split}
\Delta\mathbf{R}_k  & = \mathbf{R}_k\hat{\mathbf{R}}_k^{T},  \quad
\Delta\mathbf{t}_k   =  \mathbf{t}_k -  \Delta\mathbf{R}_k \hat{\mathbf{t}}_k.
\end{split}
\end{equation}
The ATE is calculated as
\begin{equation}
\begin{split}
\text{ATE}_{\mathbf{R}}  = \left( \frac{1}{N}\sum_{k=1}^{N}\left\| \angle\left( \Delta\mathbf{R}_k \right) \right\|^{2}\right) ^{\frac{1}{2}},  \
\text{ATE}_{\mathbf{t}}  = \left( \frac{1}{N}\sum_{k=1}^{N}\left\| \Delta\mathbf{t}_k  \right\|^{2}\right) ^{\frac{1}{2}},
\end{split}
\end{equation}
where $\angle\left( \cdot\right) $ is the angle of the angle-axis representation of  $\Delta\mathbf{R}_k$.


\textbf{Results} \quad  It is clear that our algorithm is more accurate and more robust to the initialization error than  PL2PL\cite{geneva2018lips}.  PL2PL converges with a smaller number of iterations (186 iterations) than our algorithm (723 iterations) for dataset2 at noise level 3. However, PL2PL probably converges to a local minimum, as PL2PL generates bad results for this input, as shown in Fig. 5. Except for this case, our algorithm converges faster. In addition, our algorithm has lower computational complexity than PL2PL\cite{geneva2018lips}. One point-to-plane cost in (\ref{equ_cost}) only involves 2 variables. But one relative plane cost  in \cite{geneva2018lips}    involves  3 variables (2 poses and 1 plane). Thus the Jacobian matrix of  \cite{geneva2018lips} has  more non-zero items than our reduced Jacobian matrix. 
This increases the runtime  for calculating the Jacobian matrix $\mathbf{J}$,   $\mathbf{J}^{T}\mathbf{J}$ and $\mathbf{J}^{T}\bm{\delta} $. For each iteration, our algorithm is more than 2 times faster than \cite{geneva2018lips} on dataset2 and dataset3. In the minimization process, our algorithm is about 74, 49 and  107 times faster than DPT2PL on the three datasets, respectively.
 The extra computational cost for the factorization $\mathbf{C}_{ij} = \mathbf{Q}_{ij}\mathbf{M}_{ij}$ in (\ref{equ_qr}) is marginal compared to the significant gain from the optimization process. 

\section{Conclusion}
In this paper, we have studied the PBA problem. Our main contribution is to present an efficient solution for the PBA problem using the point-to-plane cost. Although the point-to-plane cost involves a large number of constraints, we find the resulting least-squares problem has  special structure.  We prove that we can use  a reduced Jacobian matrix and  residual vector  with 4 rows to replace the original Jacobian matrix and residual vector with $K_{ij}$ rows in the LM algorithm. This reduces the runtime for computing $\mathbf{J}_{ij}$, $\bm{\delta}_{ij}$, $\mathbf{J}_{ij}^{T}\mathbf{J}_{ij}$ and $\mathbf{J}_{ij}^{T}\bm{\delta}_{ij}$  by a factor of $\frac{4}{K_{ij}}$ relative to the brute-force implementation. Our experimental results show that the extra cost of the one-time  factorization is marginal compared to the significant gain from this new formulation. Furthermore, we have verified that our algorithm is faster,  more accurate, and more robust to initial errors compared to the start-of-the-art  formulation using  the plane-to-plane cost for joint poses and planes optimization\cite{geneva2018lips}.

\bibliographystyle{abbrv-doi}

\bibliography{plane_ba}
\end{document}